\documentclass{article}

\usepackage{PRIMEarxiv}

\usepackage[utf8]{inputenc} 
\usepackage[T1]{fontenc}    
\usepackage{hyperref}       
\usepackage{url}            
\usepackage{booktabs}       
\usepackage{amsfonts}       
\usepackage{nicefrac}       
\usepackage{microtype}      
\usepackage{lipsum}
\usepackage{fancyhdr}       
\usepackage{graphicx}       
\graphicspath{{media/}}     
\usepackage{subcaption}
\usepackage{natbib}
\usepackage{amsmath}
\usepackage{amssymb}
\usepackage{mathtools}
\usepackage{amsthm}
\DeclareMathOperator*{\minimize}{minimize}
\usepackage{bm}
\usepackage{comment}
\usepackage{subcaption}

\usepackage{xcolor}

\newtheorem{theorem}{Theorem}[section]

\newtheorem{lemma}[theorem]{Lemma}

\title{Confidence HNC: A Network Flow Technique for Binary Classification with Noisy Labels} 

\author{
  Dorit Hochbaum, Torpong Nitayanont \\
  Department of Industrial Engineering and Operations Research \\
  University of California, Berkeley \\
  Berkeley, CA\\
  \texttt{\{dhochbaum, torpong\_nitayanont\}@berkeley.edu} \\
}

\begin{document}
\maketitle

\begin{abstract}
We consider here a classification method that balances two objectives: large similarity within the samples in the cluster, and large dissimilarity 
between the cluster and its complement. The method, referred to as HNC or SNC, requires seed nodes, or labeled samples, at least one of which is in the cluster and at least one in the complement.  Other than that, the method relies only on the relationship between the samples.
The contribution here is the new method in the presence of noisy labels, based on HNC, called Confidence HNC, in which we introduce ``confidence weights" that allow the given labels of labeled samples to be violated, with a penalty that reflects the perceived correctness of each given label.  If a label is violated then it is interpreted that the label was noisy.
The method involves a representation of the problem as a graph problem with hyperparameters that is solved very efficiently by the network flow technique of parametric cut.
We compare the performance of the new method with leading algorithms on both real and synthetic data with noisy labels and demonstrate that it delivers improved performance in terms of classification accuracy as well as noise detection capability.
\end{abstract}

\section{Introduction}
\label{introduction}
The performance of machine learning models depends to a great extent on the data quality and, in particular, the reliability of the labels. Label noise is one of the concerning issues that has a tremendous impact on the outcome of learning methods and receives attention from researchers in the community.

Among different classes of learning methods, semi-supervised learning is a class of methods that utilize information from unlabeled data in addition to the labeled data, and they are often used in the context where labeled data is scarce or costly \citep{zhu2009introduction}.
By counterbalancing the effect of possibly noisy labeled data with information from unlabeled data, these methods also have the potential of mitigating the issue of label noise, on top of its advantage in the scenario where labeled samples are given in a limited amount. 

A particular class of semi-supervised methods that we are interested in is the class of network-flow based, or graph based, methods in which minimum cut solution of a graph representation of the data provides label prediction of unlabeled samples. Unlabeled samples assist the method through their connectivity with labeled samples, as well as that among themselves. Examples of methods in this class are the works of \cite{blum2001learning, blum2004semi, hochbaum2009polynomial}. While these methods already control the effect of labeled samples, in this work, we regulate their influence further by modifying the graph representation. We incorporate this modification into a network-flow based method called Hochbaum's Normalized Cut (HNC) \citep{hochbaum2009polynomial}.

In HNC, samples are partitioned into two sets in a way that maximizes pairwise similarities in one class while minimizing similarities between classes by solving a single minimum cut problem on a respective graph. It was shown by \cite{baumann2019comparative}, via an extensive experimental study, that HNC is competitive and improves on leading classifiers.

The new method presented here, for classification in the presence of noisy labels, preserves the goals of HNC of maximizing the intra-similarity and minimizing the inter-similarity, and in addition allows labeled data to switch their label at a penalty cost added to the objective. The penalty costs are related to {\em confidence weights} that reflect the confidence in the reliability of the labels.  A new procedure for generating confidence weights based on the parametric minimum cut procedure is introduced here as well. The higher the confidence weight, the higher the  penalty for the labeled sample to be placed in a class that does not match its label.  The new method is called {\em Confidence HNC (CHNC)}.  It is demonstrated that CHNC is a strict generalization of HNC, yet solvable very efficiently via the parametric cut method in the complexity of a single min-cut procedure.

Additionally, CHNC provides a detection mechanism for noisy labels. That is, if a labeled sample is mislabeled in spite of the penalty, it is deemed noisy.

To demonstrate the efficiency of the CHNC method, we conduct an experimental study in which CHNC is compared with three classification methods designed to deal with noisy labels: 
two deep learning methods called \textit{Co-teaching+} \citep{yu2019does} and 
\textit{DivideMix} \citep{li2020dividemix}, and \textit{Confident Learning} \citep{northcutt2021confident} which is a filtering method that can be coupled with any classifier.   These methods are selected as baselines in this work as they have been shown to be robust to varying levels of label noise. 
The experimental study is run on both synthetic and real data. It demonstrates that CHNC outperforms the benchmark methods in terms of accuracy and balanced accuracy, and exhibits competitive noise detection capability.

The paper is organized as follows: Leading methods for label noise are discussed in Section \ref{sec:overview}. Section \ref{sec:notation} provides the problem statement as well as relevant notations and concepts of the minimum cut problem. Section \ref{subsec:hnc} explains the HNC method that our work is based on, and Section \ref{subsec:chnc} introduces the CHNC method, which incorporates the confidence weights into the minimum cut-based model. We describe how the confidence weights are computed via parametric minimum cut problems in Section \ref{section:confidence}. Section \ref{sec:implementation} includes specific implementation details for this work.
Finally, Section
\ref{sec:exp} describes the experimental study in which CHNC is compared to the three benchmark methods mentioned above.

\section{Related Works} \label{sec:overview}
This section describes several relevant classification methods that handle label noise. These methods first assess the reliability of the labeled samples, then adjust the degree to which they learn from each sample based on this evaluation.

First, we discuss preprocessing methods that filter out labeled samples that are considered noisy by the methods.
Editing methods such as Editing nearest neighbors and Repeated edited nearest neighbors \citep{tomek1976experiment} are among the first filtering methods used with the kNN classifier. These methods are first trained on the given labeled set. Labeled instances whose given labels do not match the prediction of the trained models are filtered out.
Another work is Nearest Neighbor Editing Aided by Unlabeled Data (NNEAU) \citep{guan2009nearest}, where the training set is augmented by applying co-training on the unlabeled set before filtering. A more recent method is Confident Learning by \citet{northcutt2021confident} in which the joint distribution between noisy labels and true labels is estimated. Probabilistic thresholds and ranking method are then used to filter data. Their experiments showed the competence of the method in estimating accurate noise rates, leading to accurate data filtering and classification performance that is superior to many baseline methods. Other works such as \citep{luengo2018cnc,nguyen2019self,pleiss2020identifying,kim2021fine} use thresholds on certain metrics to filter labeled samples. 

While the methods discussed above completely disregard the labels of labeled samples that are deemed noisy, there are other methods that instead reduce the influence of these likely noisy labeled samples. In Support Vector Machine, the concept of fuzzy membership is introduced to a method called Fuzzy SVM \citep{lin2004training} to limit the impact of noisy samples. The likelihood that each label is correct is computed and used to scale the penalty weight in the objective function. 


In addition to the traditional machine learning methods such as SVM, there are a number of recent deep learning models that were devised to handle data that contains noisy labels by adjusting the weights of labeled samples that the models are trained on, depending on how likely they are noisy. 
Co-teaching \citep{han2018co} leverages the concept of memorization effect, which states that the model learns mostly from clean labels in the early epochs, resulting in small losses of good samples and high losses of bad samples. Hence, two networks are trained on the small-loss samples of the other network. Co-teaching+ \cite{yu2019does} further incorporates the disagreement between the two networks by training only on small-loss samples where the two networks disagree, improving the model's robustness from Co-teaching for varying noise levels as well as outperforming other baseline methods that rely on related techniques. 

Another prominent deep learning method for label noise, which does not only vary the influence of labeled samples but also refines the given labels, is DivideMix \citep{li2020dividemix}. The method consists of two neural networks and two stages.
In the first stage, samples are grouped into \textit{good} and \textit{bad} samples by the Gaussian mixture model based on their prediction losses from the two networks. Each network exchange its partition to the other to prevent overfitting. 
The second stage refines the labels of these samples in two ways: linearly combining the given labels and predicted labels of good samples, and combining the label predictions of bad samples from two networks. The refined labels are then used for training the model with a particular loss function that was proposed in the work. DivideMix was shown as a competent method for learning with label noise that improved on many other deep learning methods, including Co-teaching. Co-teaching+ was not included in that experiment. 
There are also other works in the class of method that vary their reliance on different labeled samples including the works of \citet{liu2015classification,thulasidasan2019combating,liu2020early}. 
Our method Confidence HNC, or CHNC, also controls the extent to which each labeled sample may impact the classification result, similar to methods such as Co-teaching+ and DivideMix that assign different weights to different labeled samples depending on how likely their labels are corrupted. 
CHNC first learns the reliability, or confidence weight, of each labeled sample by solving an associated parametric minimum cut. 
This set of confidence weights are incorporated subsequently in another minimum cut problem, of which the solution determines the predicted labels for unlabeled samples and also infers whether each labeled sample is noisy or not. 


\section{Preliminaries} \label{sec:notation}

\subsection{Problem statement}
We are given a dataset that consists of a set of labeled samples $L$ and a set of unlabeled samples $U$. Each sample $i \in L \cup U$ 
has a feature vector representation $x_i$.  Each labeled sample $i \in L$ has an assigned label $y_i \in \{+1,-1\}$, where $y_i=+1$ indicates the positive class, and $y_i=-1$ is the negative class. 
Denote the positive labeled samples by, $L^+=\{i \in L | y_i = +1\}$, and the negative labeled samples by $L^- := \{i \in L | y_i = -1 \}$.

In the noisy setup, some of the given labels of samples in $L$ may be corrupt. For example, sample $i$ with a given $y_i=+1$ (or $y_i=-1$) might in fact come from the negative (or positive) class but its label was recorded as positive (or negative) due to issues in the data collection process.
The goal is to predict the classes of unlabeled samples in the set $U$, despite the noise in the given labels of samples in $L$. 


\subsection{Graph representation of data and relevant notations} \label{subsec:graphnotation}

Let each sample be represented by a node in an undirected graph, $G=(V,E)$, where $V=L\cup U$ and $E$ is a set of edges connecting pairs of nodes in $V$ with similarity weight $w_{ij}$ for $[i,j] \in E$. The set of edges, $E$, may connect all pairs of samples in $V$, or it might connect only pairs that are similar enough. 

For a partition of $V$ into two disjoint sets $S$ and $\overline{S}$, where $\overline{S} = V \backslash S$, we denote the sum of edge weights that have one endpoint in $S$ and the other in $\overline{S}$ by $C(S, \overline{S}) := \sum_{[i,j]\in E, i\in S, j\in \overline{S}} w_{ij}$. 

With this notation, $C(S, S)$ is the total sum of similarities within the set $S$, which is desired to be maximized, and $C(S, \overline{S})$ 
is the total similarity between $S$ and its complement $\overline{S}$, which is aimed to be minimized. 
Let the \textit{weighted degree} of node $i$ be denoted by $d_i$ where $d_i = \sum_{[i,j] \in E, j \in V} w_{ij}$, and the \textit{volume} of the set $S$ be denoted by $d(S)=\sum_{i \in S} d_i$.

\subsection{Minimum $(s,t)$-cut in a graph}

Consider a directed graph $G=(V,A)$ with a set of vertices $V$, and a set of directed arcs $A$, where an ordered pair $(i,j) \in A$ denotes an arc pointing from node $i$ to node $j$. Each arc $(i,j)$ is assigned an arc capacity of $c_{ij}$. \textit{An $(s,t)$-graph} of $G$, referred to as $G_{st}$, is a graph with the set of vertices $V_{st} = V \cup \{s,t\}$ where the two appended nodes $s$ and $t$ are \textit{the source node} and \textit{the sink node}, respectively. $G_{st}$ has a set of arcs $A_{st} := A \cup A_{s} \cup A_{t}$ where $A_s \subseteq \{(s,i) | i \in V\}$ is the set of source-adjacent arcs and $A_t \subseteq \{(i,t) | i \in V\}$ is the set of sink-adjacent arcs. Again, each arc $(i,j) \in A_{st}$ has a capacity $c_{ij}$. 

\textit{An} $(s,t)$\textit{-cut} of the graph $G_{st}=(V_{st},A_{st})$ is a partition of nodes into two disjoint sets $(\{s\} \cup S, \{t\} \cup \overline{S})$, where $\overline{S} = V \backslash S$. For simplicity, 
we denote hereafter an $(s,t)$-cut by using the notation of $(S,\overline{S})$, omitting $\{s\}$ and $\{t\}$ in the notation.
The two sets $S$ and $\overline{S}$ in the $(s,t)$-cut $(S,\overline{S})$ are referred to as \textit{the source set} and \textit{the sink set}, respectively.

\textit{The capacity} of the $(s,t)$-cut $(S, \overline{S})$ is the sum of capacities of arcs directed from nodes in $\{s\} \cup S,$ to nodes in $\{t\} \cup \overline{S}$: $C_{st}(S, \overline{S}) := \sum_{(i,j)\in A, i \in \{s\} \cup S, j \in \{t\} \cup \overline{S}} c_{ij}$. 

\textit{A minimum $(s,t)$-cut} of the directed graph $G_{st}$ is an $(s,t)$-cut with the minimum capacity among all $(s,t)$ cuts of $G_{st}$. That is, $(S^*, \overline{S^*})$ is a minimum $(s,t)$-cut if it minimizes $C_{st}(S, \overline{S})$ with respect to $S$ and $\overline{S}$ for $S \subseteq V$ and $\overline{S} = V \backslash S$. Hereafter, we use the term \textit{minimum cut} to refer to a minimum $(s,t)$-cut.

\subsection{Nested cut property of a parametric flow graph}

We consider an $(s,t)$-graph $G_{st}=(V_{st}, A_{st})$ with arc capacities in $A_{st}$ that are functions of a parameter $\lambda$, denoted by $G_{st}(\lambda)$. 
The graph $G_{st}(\lambda)$ is called \textit{a parametric flow graph} if the capacities of the source adjacent arcs and sink adjacent arcs are non-decreasing and non-increasing, respectively, with respect to $\lambda$, (or vise versa), while the capacities of other arcs are independent of $\lambda$.

The minimum $(s,t)$-cut of a parametric flow graph $G_{st}(\lambda)$ varies as a function of $\lambda$.  Consider a list of $q$ increasing values of $\lambda$: $\lambda _1 < \lambda _2\ldots < \lambda _q$, and denote the minimum $(s,t)$-cut of $G_{st}(\lambda)$ for $\lambda = \lambda_k$ by $(S_k, \overline{S}_k)$. The sequence of minimum $(s,t)$-cuts for these increasing values of $\lambda$, $(S_1, \overline{S}_1),$ $(S_2, \overline{S}_2), \dots,$ $(S_q, \overline{S}_q)$, follows the property stated in the lemma below.

\begin{lemma}[Nested Cut Property] \label{lemma:nestedcut}\citep{gallo1989fast,hochbaum1998pseudoflow,hochbaum2008pseudoflow} Given a parametric flow graph $G_{st}(\lambda)$ and a sequence of parameter values $\lambda _1 < \lambda _2\ldots < \lambda _q$,
the corresponding minimum cut partitions, $(S_1, \overline{S}_1),$ $(S_2, \overline{S}_2), \dots,$ $(S_q, \overline{S}_q)$, satisfy 
$S_1 \subseteq S_2 \subseteq \dots \subseteq S_q$.
\end{lemma}

Our method to compute the confidence weights of labeled samples relies on the nested cut property. The method will be discussed in detail in Section \ref{section:confidence}.

\section{The HNC Method and Its Adaptation to Noisy Labels}\label{sec:hnc}

\subsection{Hochbaum's Normalized Cut (HNC)} \label{subsec:hnc}
The HNC method \citep{hochbaum2009polynomial} partitions samples in a given data set $V:=L \cup U$ into two disjoint sets $S$ and $\overline{S}$ with the goal to balance two objectives: 1) high homogeneity (or intra-similarity) within the set $S$, and 2) large dissimilarity (or small inter-similarity) between $S$ and its complement, $\overline{S}$. \citet{hochbaum2009polynomial} stated the problem as either a ratio problem or a linearized problem with a tradeoff parameter $\lambda \geq 0$ between the two objectives:
\begin{equation} \label{eq:HNC-original}
\minimize_{\varnothing \subset S \subset V} \quad C(S,\overline{S}) - \lambda \sum_{i \in S} d_i
\end{equation}

In the context of binary classification, we designate the two sets, $S$ and $\overline{S}$, to the two classes, the positive and the negative classes, respectively. For a generic dataset, we do not know a priori whether we need to put more emphasis on the homogeneity of the positive or the negative class, or both. Hence, we alter the first objective of HNC to incorporate the homogeneity of both classes. We formulate this modified optimization problem as follows:
\begin{equation} \label{eq:HNC}
\minimize_{\varnothing \subset S \subset V} \quad C(S,\overline{S}) - \alpha \; C(S,S) - \beta \; C(\overline{S},\overline{S}) 
\end{equation}
where $\alpha \geq 0$ and $\beta \geq 0$ are the tradeoff parameters between the homogeneity within the two sets and the dissimilarity between them. We show in the following lemma that (\ref{eq:HNC}) is equivalent to (\ref{eq:HNC-original}).\\

\begin{lemma}
    Problem (\ref{eq:HNC}) is equivalent to problem (\ref{eq:HNC-original}) for $\lambda = \frac{\alpha-\beta}{1+\alpha+\beta}$
\end{lemma}
\begin{proof}
Since $C(V,V) = C(S,S) + C(\overline{S}, \overline{S}) + 2 C(S,\overline{S})$, we rewrite the objective function of (\ref{eq:HNC}) as 
$C(S, \overline{S}) - \alpha \; C(S,S) - \beta \; (C(V,V)-C(S,S)-2C(S,\overline{S})) = (1+2\beta) \; C(S,\overline{S}) - (\alpha-\beta) \; C(S,S) - \beta \; C(V,V)$. Since $C(V,V)$ is a constant, independent of $S$, minimizing this objective function is equivalent to minimizing 
\begin{equation} \tag{\ref{eq:HNC}$'$}\label{eq:HNC-rewrite}
    C(S,\overline{S}) - (\frac{\alpha-\beta}{1+2\beta}) \; C(S,S)
\end{equation}

Since $C(S,S) = \sum_{i \in S}\sum_{j\in S} w_{ij} = \sum_{i \in S} (\sum_{j \in V} w_{ij} - \sum_{j \in \overline{S}} w_{ij}) = \sum_{i \in S} d_i - C(S,\overline{S})$, by substituting this expression in (\ref{eq:HNC-rewrite}), we can see that solving (\ref{eq:HNC}) is equivalent to 
\begin{equation} \label{eq:HNC-lambda}
\minimize_{\varnothing \subset S \subset V} \quad C(S,\overline{S}) - \lambda \sum_{i \in S} d_i
\end{equation}
where $\lambda = \frac{\alpha-\beta}{1+\alpha+\beta}$. This formulation is the same as the one proposed by \citet{hochbaum2009polynomial}.  \hfill \qedsymbol
\end{proof}

Notice that when we put more emphasis on the homogeneity of $S$, over that of $\overline{S}$, $\alpha \geq \beta$ and we have $\lambda = \frac{\alpha-\beta}{1+\alpha+\beta} \geq 0$. On the other hand, when the homogeneity of $\overline{S}$ is given more weight, we have $\lambda < 0$.

To avoid trivial solutions where the cluster $S$ is either the entire set or the empty set, the method requires at least one seed node in the cluster and at least one in the complement.  Under the name SNC (supervised normalized cut) or HNC (Hochbaum's Normalized Cut), the method used the labeled samples as the respective seed nodes in binary classification \citep{yang2014SNC-nuclear, baumann2019comparative, spaen-eNeuro2019hnccorr, asin-neuroAOR2020hnccorr}.

With the set $S$ designated for the positive class and $\overline{S}$ for the negative class, we use the positive and negative labeled samples as seed nodes to \textit{supervise} the partition by forcing them to belong to $S$ and $\overline{S}$, respectively. This is formulated by adding the restrictions that $L^{+}$ is in $S$ and $L^{-}$ is in $\overline{S}$: 
\begin{equation} \label{eq:Classification-HNC}
    \minimize_{ L^{+} \subseteq S \subseteq V\setminus L^{-} } \; C(S, \overline{S}) - \lambda \sum_{i \in S} d_i
\end{equation}

For an appropriately chosen value of $\lambda$, let $S^*$ be the optimal solution for (\ref{eq:Classification-HNC}) and $\overline{S^*}$ be its complement, then unlabeled samples in $S^*$, or $S^* \cap U$, are predicted as positive whereas other unlabeled samples, $\overline{S^*} \cap U$, are predicted negative. We will refer to the set $S^*$ as \textit{positive prediction set} and the set $\overline{S^*}$ as the \textit{negative prediction set}.

Not only are problems (\ref{eq:HNC-lambda}) and (\ref{eq:Classification-HNC}) solved efficiently for a given value of $\lambda$ as a minimum $(s,t)$-cut problem on an associated graph, but they are solved efficiently for {\em all} values of $\lambda$, using the parametric minimum cut algorithm of \cite{hochbaum2008pseudoflow}, in the complexity of a single minimum $(s,t)$-cut procedure, \cite{hochbaum2009polynomial, hochbaum2013polynomial}. This efficient parametric cut procedure is used here for tuning our implementation and selecting the ``best" value of $\lambda$. 

The construction of the associated graph, $G_{st}(\lambda)$ is given by \citet{hochbaum2009polynomial} and applied for problem (\ref{eq:Classification-HNC}) for $\lambda \geq 0$. We first explain the graph for the case of $\lambda \geq 0$. After that, we describe a slight modification to obtain the graph associated with problem (\ref{eq:Classification-HNC}) for a general value of $\lambda$, whether positive or negative.

In the graph representation of the data, $G=(V,E)$, introduced in Section \ref{sec:notation}, we now replace each edge $[i,j] \in E$  by two directed arcs, $(i,j)$ and $(j,i)$. Let $A$ denote this set of arcs, that is, $A = \{(i,j), (j,i) | [i,j] \in E\}$. Arcs $(i,j)$ and $(j,i)$ have the same capacities $c_{ij}=c_{ji}=w_{ij}$.

The graph $G_{st}(\lambda)=(V_{st}, A_{st})$ is a directed graph where
$V_{st}$ is formed by adding a source node $s$ and a sink node $t$ to the set of nodes $V$; $A_{st} = A \cup A_s \cup A_t$ where $A_s = \{(s,i) | i \in L^+ \cup U\}$ and $A_t = \{(i,t) | i \in L^-\}$. For $i \in L^+$, the arc capacities $c_{si}$ are set to be infinite. Similarly, for $i \in L^-$, the arc capacities $c_{it}$ are infinite.
These infinite capacity arcs guarantee that all nodes of $L^{+}$ are included in  the source set of any finite cut, and all nodes of $L^{-}$ are included in the sink set of any finite cut.
For each $i \in U$, the arc $(s,i)$ has a capacity of $\lambda d_i$ where $d_i=\sum_{(i,j) \in A} w_{ij}$. The graph $G_{st}(\lambda)$ is illustrated in Figure \ref{fig:HNC-s}.

It was proved in \cite{hochbaum2009polynomial, hochbaum2013polynomial} that a minimum cut partition $(S^*, \overline{S^*})$ in $G_{st}(\lambda)$ is an optimal solution to problem (\ref{eq:Classification-HNC}) for $\lambda \geq 0$.
This result is generalized in the following lemma, where we introduce another graph, $G_{st}^{'}(\lambda)$, illustrated in Figure \ref{fig:HNC-st}, 
in which the minimum $(s,t)$-cut $(S^*, \overline{S^*})$ is the optimal solution to Problem (\ref{eq:Classification-HNC}) for any $\lambda$, whether $\lambda \geq 0$ or $\lambda < 0$.

\begin{lemma} \label{lemma:new-graph}
Let $G_{st}^{'}(\lambda)$ be a directed graph with the set of vertices $V_{st}$ and the set of arcs $A_{st} = A \cup A_s \cup A_t$. $V_{st}$ and $A$, as well as their arc capacities, are defined similarly to those in $G_{st}(\lambda)$. $A_s = \{(s,i)| i \in L^+ \cup U\}$ and $A_t = \{(i,t) | i \in L^- \cup U\}$. For $i \in L^+$ and $j \in L^-$, $c_{si} = \infty$ and $c_{jt} = \infty$. For $i \in U$, $c_{si} = \max(\lambda d_i, 0)$ and $c_{it}=\max(-\lambda d_i, 0)$.
For any $\lambda$, if $(S^*, \overline{S^*})$ is the minimum $(s,t)$-cut of $G_{st}^{'}(\lambda)$, then $(S^*, \overline{S^*})$ is also the optimal solution to problem (\ref{eq:Classification-HNC}) for that value of $\lambda$.
    
\end{lemma}
\begin{proof}
The graph $G_{st}^{'}(\lambda)$ is displayed in Figure \ref{fig:HNC-st}.

As noted previously, for $\lambda \geq 0$, $G_{st}^{'}(\lambda)$ is equivalent to $G_{st}(\lambda)$. The statement in the lemma for this case was proved in \cite{hochbaum2009polynomial, hochbaum2013polynomial}.  Therefore, it remains to address the case of $\lambda < 0$.

For $\lambda < 0$, in $G_{st}^{'}(\lambda)$, $c_{si} = 0$ and $c_{it} = |\lambda|d_i$ for each $i \in U$. We consider a particular finite $(s,t)$-cut $(S, \overline{S})$ of $G_{st}^{'}(\lambda)$. This cut's capacity is equal to $\sum_{i \in \{s\} \cup S} \sum_{j \in \{t\} \cup \overline{S}} c_{ij} = \sum_{i \in S} \sum_{j \in \overline{S}} c_{ij} + \sum_{i \in S} c_{it} + \sum_{j \in \overline{S}} c_{sj}$. Consider the three terms in this expression. 

The first term is $\sum_{i \in S} \sum_{j \in \overline{S}} c_{ij} = \sum_{i \in S} \sum_{j \in \overline{S}} w_{ij}$. Since $S$ consists of two disjoint subsets $L^+$ and $S \cap U$, the second term can be rewritten as
\begin{align*}
    \sum_{i \in S} c_{it} &= \sum_{i \in L^+} c_{it} + \sum_{S \cap U} c_{it} = \sum_{i \in L^+} 0 + \sum_{i \in S \cap U} |\lambda| d_i\\
    &= (\sum_{i \in L^+}|\lambda| d_i  + \sum_{i \in S \cap U} |\lambda| d_i) - \sum_{i \in L^+}|\lambda| d_i\\
    &= \sum_{i \in S} |\lambda| d_i - \sum_{i \in L^+}|\lambda| d_i
\end{align*}

The third term is $\sum_{j \in \overline{S}} c_{sj}$ which is equal to zero because $\overline{S}$ consists of $L^-$ and $\overline{S} \cap U$, and nodes in $L^-$ are not connected to $s$ while $c_{sj} = 0$ for $j \in U$.

Hence, the capacity of the cut is equal to
$\sum_{i \in S} \sum_{j \in \overline{S}} w_{ij} + \sum_{i \in S} |\lambda|d_i - \sum_{i \in L^+}|\lambda| d_i = C(S, \overline{S}) + \sum_{i \in S} |\lambda|d_i - \sum_{i \in L^+}|\lambda| d_i$, which is equal to $C(S, \overline{S}) - \sum_{i \in S} \lambda d_i + \sum_{i \in L^+}\lambda d_i$ when $\lambda < 0$. Notice that the last term, $\sum_{i \in L^+}\lambda d_i$, is constant as it is independent of the partition $(S, \overline{S})$. Hence, minimizing this cut's capacity is equivalent to minimizing $C(S, \overline{S}) - \sum_{i \in S} \lambda d_i$, which is the objective function of (\ref{eq:Classification-HNC}).

Hence, given  a minimum cut in $G_{st}^{'}(\lambda)$, $(S^*, \overline{S^*})$, the partition $(S^*, \overline{S^*})$ is the optimal solution to Problem (\ref{eq:Classification-HNC}). \hfill \qedsymbol


\end{proof}

\begin{figure}
    \centering
    \includegraphics[width=0.5\linewidth]{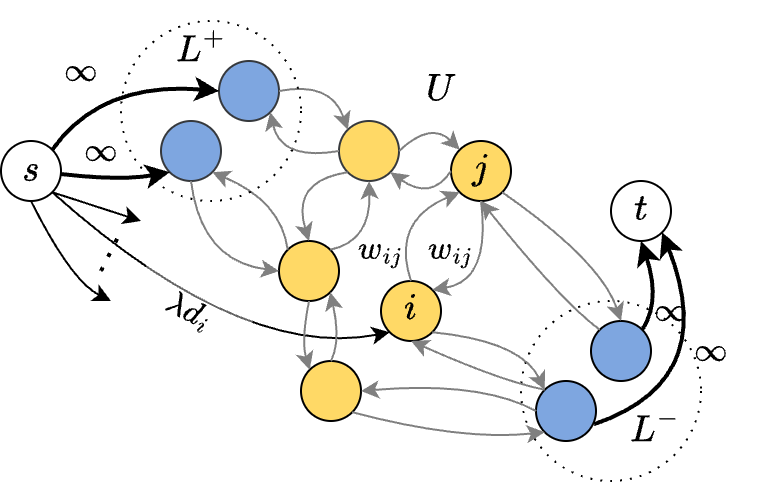}
    \caption{Associated graph $G_{st}(\lambda)$ whose minimum cut provides a solution for (\ref{eq:Classification-HNC}) when $\lambda \geq 0$.}
    \label{fig:HNC-s}
\end{figure}
\begin{figure}
    \centering
    \includegraphics[width=0.5\linewidth]{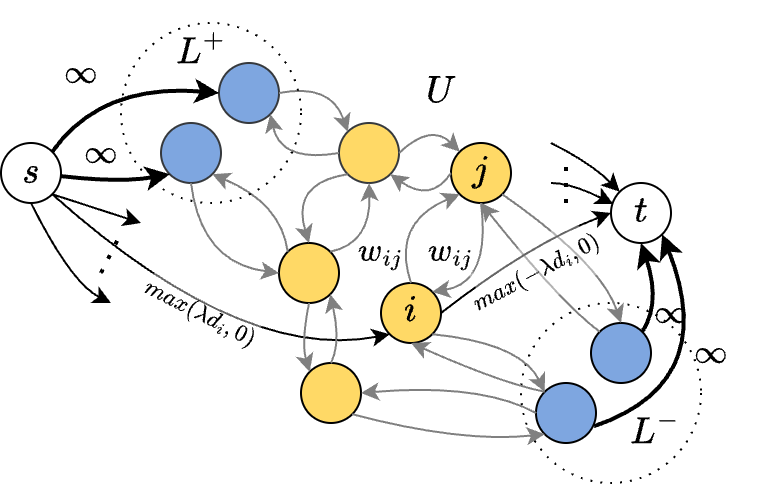}
    \caption{Associated graph $G^{'}_{st}(\lambda)$ whose minimum cut provides a solution for (\ref{eq:Classification-HNC}), generalized for both $\lambda > 0$ and $\lambda < 0$}
    \label{fig:HNC-st}
\end{figure}

We remark here that there are no separate training and testing steps in HNC. HNC partitions the entire set of samples $L \cup U$ together at the same time. 
The prediction of unlabeled samples in $U$ is based on the partition result. This differentiates HNC, as a transductive method, from inductive methods in which a classifier is trained on a given data in the training step and used for prediction of any unseen unlabeled samples in a later stage while HNC provides predictions specifically for unlabeled samples in $U$.




\subsection{Confidence HNC (CHNC)} \label{subsec:chnc}
The idea of CHNC is to allow labeled samples to belong to the opposite set with penalty.
Instead of serving as a seed node for the positive or the negative prediction set, there is a 
confidence weight, $\gamma_i$, for labeled sample $i$, that reflects how {confident} we are in the label quality of sample $i$. 
(The confidence weights computation is explained in Section \ref{section:confidence}.) Before we introduce the formulation of CHNC, we first consider an alternative, but equivalent, formulation of HNC (\ref{eq:Classification-HNC}) in the following lemma.

\begin{lemma} Problem (\ref{eq:Classification-HNC}) is equivalent to
\begin{equation} \label{eq:Classification-HNC1}
    \minimize_{ L^{+} \subseteq S \subseteq V\setminus L^{-} } \; C(S, \overline{S}) - \lambda \sum_{i \in S\cap U} d_i 
\end{equation} 
\end{lemma}

\begin{proof}
$S$ is the union of $S \cap U$, $S \cap L^+$ and $S \cap L^-$, which are disjoint. $S \cap L^+ = L^+$ since $L^+ \subseteq S$ and $S \cap L^- = \varnothing$ since $S \subseteq V \backslash L^-$.  Thus, $\sum_{i \in S} d_i = \sum_{i \in S\cap U} d_i + \sum_{i \in L^+} d_i$. As $L^+$ is given, $\sum_{i \in L^+} d_i$ is a constant. Hence, minimizing $- \lambda \sum_{i \in S} d_i$ is equivalent to minimizing $- \lambda \sum_{i \in S \cap U} d_i$.  \hfill \qedsymbol
\end{proof}

In problem (\ref{eq:Classification-HNC}), and (\ref{eq:Classification-HNC1}), the constraints $L^{+} \subseteq S \subseteq V\setminus L^{-}$ can be incorporated in the objective function by giving infinite penalty for violating them. Formally, we assign penalty weights $\gamma_i$ to all $i\in L$.
For penalty weights $\gamma_i$ equal to $\infty$, or a very large number that is an upper bound on the objective value, problem  (\ref{eq:Classification-HNC1}) is written as,
\begin{equation}
\minimize_{ \emptyset \subset S \subset V } \; C(S, \overline{S}) - \lambda \sum_{i \in S\cap U} d_i + \sum_{i \in \overline{S} \cap L^{+}} \gamma_i +  \sum_{j \in S\cap L^{-}} \gamma_j.
\label{eq:CHNC}
\end{equation}

The CHNC problem that we consider here is (\ref{eq:CHNC}) when the penalty weights $\gamma_i$ are allowed to be finite, and therefore the solution may violate some of the constraints $L^{+} \subseteq S \subseteq V\setminus L^{-}$.  In that sense, HNC is a special case of CHNC.





The associated graph $G_{st}^{c}(\lambda)$ whose minimum $(s,t)$-cut provides an optimal solution for (\ref{eq:CHNC}) 
is given in Figure \ref{fig:chnc-graph}. 
The difference between the graph $G_{st}^{'}(\lambda)$ for HNC and $G_{st}^{c}(\lambda)$ for CHNC  is in replacing the infinite capacity weights of arcs in $\{(s,i) | i \in L^+\}$ and $\{(i,t) | i \in L^-\}$ by the \textit{confidence weights}, $\gamma_i$ for node $i$.
We next show that, indeed, the minimum $(s,t)$-cut of $G_{st}^c(\lambda)$ provides the optimal solution to the CHNC problem (\ref{eq:CHNC}).

\begin{figure}
    \centering
    \includegraphics[width=0.5\linewidth]{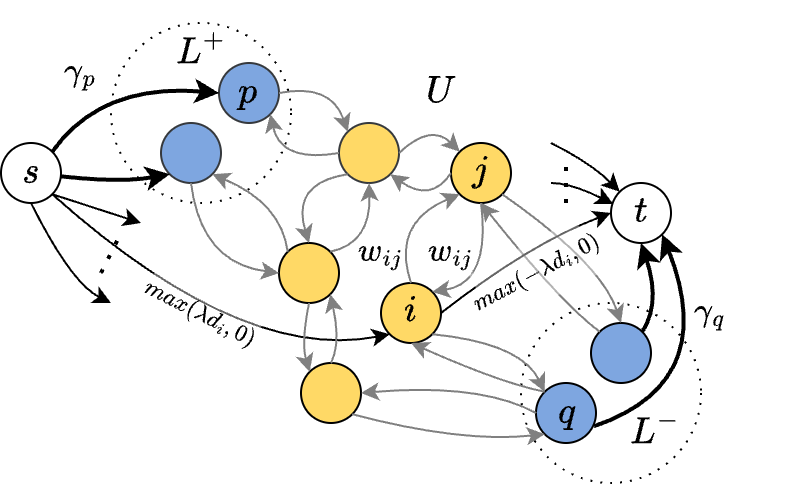}
    \caption{Associated graph $G_{st}^c(\lambda)$ for CHNC (\ref{eq:CHNC})}
    \label{fig:chnc-graph}
\end{figure}

\begin{theorem}\label{thm:CHNC-cut}
Let $(S^*, \overline{S^*})$ be a minimum $(s,t)$-cut for the Confidence HNC graph $G^c_{st}(\lambda)$. Then, $(S^*, \overline{S^*})$ is an optimal solution to CHNC (\ref{eq:CHNC}).
\end{theorem}
\begin{proof}
Let ($S, \overline{S}$) be a finite $(s,t)$-cut on $G^c_{st}(\lambda)$. 
The capacity of this cut is given by

\begin{equation*}
\begin{aligned}
    C_{st}(S, \overline{S}) = &\sum_{i\in S} \sum_{j\in \overline{S}} w_{ij} + \sum_{j\in U\cap \overline{S}} \max(\lambda d_j,0) + \sum_{j\in U\cap S} \max(-\lambda d_j,0) \\
    &+ \sum _{i\in L^+\cap \overline{S}} \gamma_i + \sum _{k \in L^- \cap S} \gamma_k
    \end{aligned}
\end{equation*}

The sum of the second and third terms can be written as
\begin{equation*}
\begin{aligned}
&\sum_{j\in U\cap \overline{S}} \max(\lambda d_j,0) + \sum_{j\in U\cap S} \max(-\lambda d_j,0)\\
&= \sum_{j\in U\cap \overline{S}} (\lambda d_j + \max(-\lambda d_j, 0)) + \sum_{j\in U\cap S} \max(-\lambda d_j,0) \\
&= \sum_{j\in U\cap \overline{S}} \lambda d_j + \sum_{j \in U} \max(-\lambda d_j, 0)
\end{aligned}
\end{equation*}
where $\sum_{j \in U} \max(-\lambda d_j, 0)$ is a constant since $U$ is given. 

Moreover, $\sum_{j\in U\cap \overline{S}} \lambda d_j = \sum_{j \in U} \lambda d_j - \sum_{j\in U\cap S} \lambda d_j$. Again, the first term, $\sum_{j \in U} \lambda d_j$, is a constant. Hence, $\sum_{j\in U\cap \overline{S}} \max(\lambda d_j,0) + \sum_{j\in U\cap S} \max(-\lambda d_j,0)$ is a sum of $- \sum_{j\in U\cap S} \lambda d_j$ and a constant. Therefore, minimizing $C_{st}(S, \overline{S})$, as expressed at the beginning of the proof, is equivalent to minimizing 
\[C(S,\overline{S}) -\lambda\sum_{j\in U\cap S} d_j + \sum _{i\in L^+\cap \overline{S}} \gamma_i + \sum _{k \in L^- \cap S} \gamma_k\]
which is the CHNC problem (\ref{eq:CHNC}).
\hfill \qedsymbol
\end{proof}

Therefore, we find the optimal $S^*$ for (\ref{eq:CHNC}) by solving for the minimum $(s,t)$-cut of the graph $G^c_{st}(\lambda)$, which is the associated graph of CHNC. Unlabeled samples that are in $S^{*}$ are predicted positive and other unlabeled samples, those in $\overline{S^*}$, are predicted negative.

To determine the value of the parameter $\lambda$, we use the $k$-fold cross validation, applied on the labeled set $L$, to select the ``best'' value of $\lambda$ that yields the best validation accuracy. 

\section{Confidence Weights Computation}\label{section:confidence}

We describe here the method for computing the confidence weights of labeled samples in $L:=L^+ \cup L^-$. The method consists of two separate stages: one to compute the confidence weights of positive labeled samples in $L^+$ and the other to compute the confidence weights of negative labeled samples in $L^-$. Each of the two steps involves solving a problem similar to HNC (\ref{eq:Classification-HNC}) for a sequence of values of the parameter $\lambda$ and relies on the nested cut property, Lemma \ref{lemma:nestedcut}. 

First, we remark that the graph associated to HNC (\ref{eq:Classification-HNC}), or graph $G_{st}^{'}(\lambda)$, is a parametric flow graph, which is defined in Section \ref{sec:notation}. It follows that the nested cut property applies to $G_{st}^{'}(\lambda)$. 
To see that $G_{st}^{'}(\lambda)$ is a parametric flow graph, note that the source adjacent arcs $(s,i)$ have capacities of $\max(\lambda d_i, 0)$ for $i \in U$ or $\infty$ for $i \in L^+$. In either case, it is non-decreasing with respect to $\lambda$. The sink adjacent arcs $(i,t)$ have capacities $\max(-\lambda d_i, 0)$ for $i \in U$ or $\infty$ for $i \in L^-$, which are non-increasing as a function of $\lambda$. Other arc capacities are independent of $\lambda$. Therefore, $G_{st}^{'}(\lambda)$ is a parametric flow graph.

The graphs that we will consider in both stages of the confidence weights computation are special cases of $G_{st}^{'}(\lambda)$. Hence, these graphs are parametric flow graphs, and the nested cut property applies to both of them.

\subsection{Confidence weights computation of positive labeled samples}\label{subsec:confw-pos}
To evaluate the reliability of the labels of samples in $L^+$, which are given as positive, in this stage, we treat them as unlabeled and solve for their \textit{new} labels via the HNC problem. Formally, we consider a problem instance where the set $L^-$ is the seed set for the negative prediction set $\overline{S}$, and there is \textit{no} seed set for the positive prediction set $S$. $L^+$ and $U$ together form the new unlabeled set. The HNC problem considered in this stage is 

\begin{equation} \label{eq:compute-positive-confidence}
    \minimize_{ S \subseteq V\setminus L^{-} } \; C(S, \overline{S}) - \lambda \sum_{i \in S} d_i
\end{equation}

The associated graph to (\ref{eq:compute-positive-confidence}) is the graph $G^{+}(\lambda)$, displayed in Figure \ref{fig:HNC-positive}. $G^{+}(\lambda)$ is a special case of $G_{st}^{'}(\lambda)$, and hence, it is a parametric flow graph. As a result, the nested cut property applies to $G^{+}(\lambda)$. 

\begin{figure}
     \centering
     \begin{subfigure}{0.49\textwidth}
         \centering
         \includegraphics[width=\linewidth]{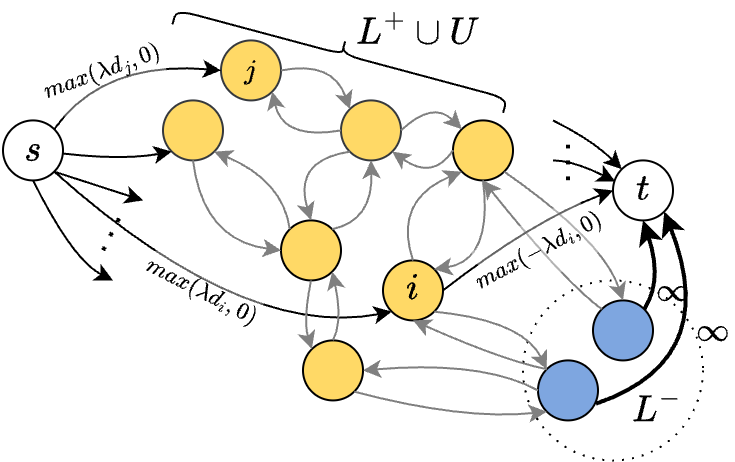}
         \caption{$G^{+}(\lambda)$ for the computation of confidence weights of positive labeled samples in $L^+$}
         \label{fig:HNC-positive}
     \end{subfigure} \hfill
     \begin{subfigure}{0.49\textwidth}
         \centering
         \includegraphics[width=\linewidth]{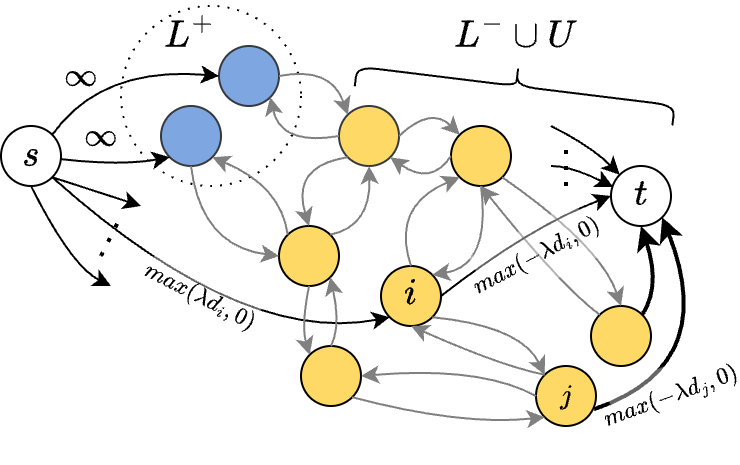}
         \caption{$G^{-}(\lambda)$ for the computation of confidence weights of negative labeled samples in $L^-$}
         \label{fig:HNC-negative}
     \end{subfigure}
     \caption{Associated graphs for the confidence weights computation}
     \label{fig:compute-confidence}
\end{figure}

Suppose we are given a sequence of $q$ increasing values of $\lambda$: $\lambda_1 < \lambda_2 < \dots < \lambda_q$. Denote the optimal solution to (\ref{eq:compute-positive-confidence}) for $\lambda_k$ by $(S_k, \overline{S}_k)$. The nested cut property implies $S_1 \subseteq S_2 \subseteq \dots \subseteq S_q$. For a sufficiently small value (a highly negative value) of $\lambda_1$, $(S_1, \overline{S}_1) = (\varnothing, V)$. As $\lambda$ increases, the positive prediction set $S$ expands in a nested manner. Each sample $i$ in $L^+ \cup U$ moves from $\overline{S}_{q_i}$ to $S_{q_i+1}$ for a certain $q_i \in \{1,\dots,q\}$. With a sufficient large value $\lambda_q$, $(S_q, \overline{S}_q) = (L^+\cup U, L^-)$.


Consider two positive labeled samples, $i$ and $j$, in $L^+$ such that $q_i < q_j$, where $q_i$ is defined as $\max \{k \; | \; i \in \overline{S}_k, 1\leq k \leq q\}$. By this definition, in the sequence of optimal partitions corresponding to $q$ increasing values of $\lambda$, sample $i$ belongs to $\overline{S}_1, \overline{S}_2, \dots, \overline{S}_{q_i}, S_{q_i+1} \dots S_{q_j}, S_{q_j+1} \dots S_q$. On the other hand, sample $j$ belongs to $\overline{S}_1, \overline{S}_2, \dots, \overline{S}_{q_i}, \overline{S}_{q_i+1} \dots \overline{S}_{q_j}, S_{q_j+1} \dots S_q$. The two samples belong to different sets in the optimal partitions $(S_k, \overline{S}_k)$ for $k \in [q_i+1, q_j]$, in which sample $i$ always belongs to the positive prediction set, $S_k$, and sample $j$ belongs to the negative prediction set, $\overline{S}_k$. We infer from this distinction that for two positive labeled samples $i$ and $j$ where $q_i < q_j$, the sample $i$ is more likely to be positive than sample $j$. That is, the positive label of sample $i$ is more reliable than that of sample $j$.

Based on this, we compute the confidence weight of a positive labeled sample $i \in L^+$ as $\frac{|\overline{S}_{q_i}|}{|L^+ \cup U|}$. This confidence weight varies between $0$ and $1$. For samples $i$ and $j$ such that $q_i < q_j$, the proposed confidence weight of $i$ is larger than that of $j$ since $|\overline{S}_{q_i}| > |\overline{S}_{q_j}|$.

Problem (\ref{eq:compute-positive-confidence}) is solved for the sequence $\lambda_1 < \lambda_2 < \dots < \lambda_q$ as a parametric minimum cut problem using the pseudoflow algorithm of \citep{hochbaum2008pseudoflow}. We discuss in Section \ref{sec:implementation} about a specific implementation of the pseudoflow algorithm as well as the values of $\lambda_1, \lambda_2, \dots, \lambda_q$ that we use in this work.

\subsection{Confidence weights computation of negative labeled samples}

To compute the confidence weights of negative labeled samples in $L^-$, we apply a procedure similar to Section \ref{subsec:confw-pos} with $L^+$ taking the role of $L^{-}$ and vice versa. That is, we solve the HNC problem where $L^+$ is the seed set for the positive prediction set $S$ and there is no seed set for the negative prediction set $\overline{S}$. Negative labeled samples in $L^-$ and unlabeled samples in $U$ together form the new unlabeled set. 

\begin{equation} \label{eq:compute-negative-confidence}
    \minimize_{L^+ \subseteq S \subseteq V} \; C(S, \overline{S}) - \lambda \sum_{i \in S} d_i
\end{equation}

The graph associated with (\ref{eq:compute-negative-confidence}) is $G^{-}(\lambda)$, displayed in Figure \ref{fig:HNC-negative}. The nested cut property also applies to $G^{-}(\lambda)$ since it is a parametric flow graph. For a given sequence of values of $\lambda$: $\lambda_1 < \lambda_2 < \dots < \lambda_q$, where $\lambda_1$ is sufficiently small and $\lambda_q$ is sufficiently large, the corresponding sequence of optimal solution begins with $(S_1, \overline{S}_1) = (L^+, L^- \cup U)$ and ends with $(S_q, \overline{S}_q) = (V, \varnothing)$. Between the two solutions is a nested sequence of minimum $(s,t)$-cut solutions.

Consider two negative labeled samples $i$ and $j$ in $L^-$ such that $q_i < q_j$ where $q_i:=\max \{k \; | \; i \in \overline{S}_k, 1\leq k \leq q\}$. Samples $i$ and $j$ belong to the same partition of the optimal solutions to (\ref{eq:compute-negative-confidence}) for all tradeoff values, except for $\lambda_k$ for all $k$ such that $q_i \leq k < q_j$. For such $\lambda_k$, sample $i$ belongs to the positive prediction set $S_k$, while sample $j$ remains in another partition, $\overline{S}_k$, that is the negative prediction set. This implies that sample $j$ is more likely to be negative than sample $i$. Hence, for sample $j$ with a larger value of $q_j$, we have a higher confidence in the given negative label of sample $j$.  Based on this rationale, the confidence weight of a negative labeled sample $i \in L^-$ is computed as $\frac{|S_{q_i}|}{|L^- \cup U|}$, which increases as $q_i$ goes up. 

\subsection{Scaling confidence weights}
The confidence weights are computed as explained in the procedures above: the confidence weight of sample $i$ is $\gamma_i = \frac{|\overline{S}_{q_i}|}{|L^+ \cup U|}$ for $i \in L^+$ and $\gamma_i = \frac{|S_{q_i}|}{|L^- \cup U|}$ for $i \in L^-$. These values are in the range of $[0,1]$.

The confidence weights are used as arc capacities for arcs that connect labeled samples and either the source or the sink node, in the graph $G_{st}^c(\lambda)$ in Figure \ref{fig:chnc-graph}.
However, the magnitude of other arcs in the graph is not necessarily in the range of $[0,1]$. We then scale the confidence weights accordingly by a factor of $\theta$ that is equal to the average of $\{w_{ij} \; | \; (i,j) \in A\}$. Hence, the scaled confidence weights of a labeled sample $i$ is $\gamma_i = \theta \cdot \frac{|\overline{S}_{q_i}|}{|L^+ \cup U|}$ for $i \in L^+$ and $\gamma_i = \theta \cdot \frac{|S_{q_i}|}{|L^- \cup U|}$ for $i \in L^-$.

\section{Implementation Details of Confidence HNC}\label{sec:implementation}

This section provides the specification of several implementation details. First, we give a description of the parametric minimum cut solver used in both the confidence weights computation and the tuning of the parameter $\lambda$ in solving CHNC (\ref{eq:CHNC}). Second, we set up the graph representation of the data, provide a motivation for using graph sparsification, and show how it is achieved with the k-nearest neighbor sparsification.  Third, we describe the pairwise similarity measure between samples that we used in this work.

\subsection{Solving Parametric Minimum Cut Problems}

\citet{hochbaum2008pseudoflow} provided two versions of the parametric pseudoflow algorithm for solving the parametric minium cut problem.
The first version is a \textit{fully} parametric cut solver that identifies all values of $\lambda$ that give different minimum cut solutions.  Because of the nestedness property, there can be at most $n=|V|$ different such values. In the other version, the \textit{simple} parametric cut solver, the minimum $(s,t)$-cut is found for a given list of $\lambda$ values. Both versions are ``parametric" in that they have the complexity of a single minimum $(s,t)$-cut procedure. 
In this work, we use the implementation\footnote{\url{https://riot.ieor.berkeley.edu/Applications/Pseudoflow/parametric.html}} that is a simple parametric cut solver. 
The list of $\lambda$ values that we use is $-1, -0.998, -0.996,$ $\cdots, -0.004, -0.002, 0,0.002, $ $0.004, \cdots, 0.996, 0.998, 1$. 

The parametric minimum cut procedure is also used for 
solving the CHNC problem (\ref{eq:CHNC}) after the confidence weights $\{\gamma_i\}_{i=1}^n$ have been obtained. In CHNC (\ref{eq:CHNC}), $\lambda$ is an unknown parameter. In order to pick a ``best'' $\lambda$, we perform 5-fold cross validation to select it from a given parameter candidates list. In order to evaluate different values of $\lambda$ in the list, we solve the problem for all given $\lambda$s as the parametric minimum cut problem using the implementation mentioned above, in the complexity of a single min-cut procedure.

\subsection{Graph Sparsification}
The default setting of the graph for a general data set would include similarity weights for all pairs of samples in $L \cup U$ resulting in a complete graph. However, minimum cut problem on dense graphs tend to favor highly unbalanced partitions due to the small number of edges that connect a small set in the unbalanced partition to its complement. 
To see this, consider a balanced cut in a complete graph of $n$ vertices that contains half of the vertices in the source set and the other half in the sink set. The number of edges between the two sets in the partition is $\frac{n^2}{4}$. Even if the capacities on these edges are small, their sum may accumulate to a cut capacity that is much larger than that of an unbalanced cut, especially when $n$ is large. 

To avoid this issue, we apply the k-nearest neighbor sparsification, which has been used in prior works such as \citet{blum2001learning}, \citet{blum2004semi} and \citet{wang2013semi}. Under th kNN sparsification, nodes $i$ and $j$ are connected if $i$ is among the $k$ nearest neighbors of $j$, or $j$ is among those of $i$.

Another benefit of sparsification is that it reduces the size of the graph in terms of the number of arcs and, hence, the runtime.

In this work, we use $k=15$ on all data of size smaller than $10000$ samples. For larger data size, we use a smaller value of $k$, $k=10$, to enhance the effect of the sparsification, which is needed on a graph with large $n$.

\subsection{Pairwise Similarities Computation}

The pairwise similarities used as edge weights are set to depend on the distance between the feature vectors of the respective samples. Given the distance between samples $i$ and $j$, $dist(i,j)$, we use the Gaussian weight $w_{ij} = exp(-\frac{dist(i,j)}{2\sigma^2})$.
Gaussian weight is commonly used, e.g. in \citet{chen2009similarity, zhu2009introduction}.
For data of size smaller than $10000$, we use $\sigma=0.75$. For larger datasets, we use a smaller value of $\sigma = 0.5$ to bring the similarity weights for large pairwise distances closer to zero, for the same purpose as the graph sparsification.

To compute the distance, we use the weighted Euclidean distance where features are weighted according to \textit{the feature importance} obtained from the random forest classifier \citep{breiman2001random} as implemented on Scikit-Learn \citep{scikit-learn}. This step is incorporated so that the computed pairwise similarity reflects relationship between samples more accurately. 

All parameters in the random forest model are set to their default values, except for \textit{min\_samples\_leaf}, or the minimum fraction of samples required in the tree node to be split, that is tuned, via 5-fold cross validation. The values of \textit{min\_samples\_leaf} that are included in cross validation are $0.001, 0.002, 0.005$ and $0.01$. The importance of each feature given by the random forest is computed based on the impurity reduction attributed to that feature.  In other words, a feature with high importance is highly effective at distinguishing samples from different classes.

Consider a dataset whose feature representations of the samples are of $H$ dimensions. Denote the feature importance from the random forest classifier of the $h$-th feature by $\rho_h \in [0,1]$ for $h \in H$. The distance between samples $i$ and $j$, or $dist(i,j)$, is computed as $dist(i,j) = \sqrt{\sum_{h=1}^H \rho_h (x_{ih}-x_{jh})^2}$ where $x_{ih}$ is the value of the $h$-th feature of samples $i$.
Note that the importances of all features are normalized so that $\sum_{h=1}^H \rho_h = H$, to make it comparable to the unweighted distance in which $\rho_h =1$ for all $h$ and the sum of the importance weights across all features is equal to $H$.


\section{Experiments} \label{sec:exp}

To evaluate our model, we compare its classification performance and noise detection capability with three other classification methods on both synthetic and real datasets.

\subsection{Benchmark Methods}

We compare Confidence HNC with the benchmark methods:  DivideMix (DM) \citep{li2020dividemix}; Confident Learning (CL) \citep{northcutt2021confident}; and Co-teaching+ (CT+) \citep{yu2019does}. These methods have been widely used as benchmarks in works on classification with label noise and have shown competitive results.

We use the available implementations of DivideMix \footnote{\url{https://github.com/LiJunnan1992/DivideMix}}, Confident Learning\footnote{\url{https://docs.cleanlab.ai/stable/index.html}} and Co-teaching+\footnote{\url{https://github.com/xingruiyu/coteaching_plus}}. The implementation of Confident Learning is part of a Python library called Cleanlab.
For these models,
we use the default choices of hyperparameters with some exceptions, when alternative parameter values that we tried produced better results: For DivideMix, we use the default learning rate of 0.02, except for Phishing and Adult data where we use the learning rate of 0.002. The batch size of 32 is used on smaller datasets (fewer than 10000 samples) and the batch size of 128 are used on larger data. The default number of epochs is 300 in the available implementation but we observed that on our data, the accuracy of DivideMix on both training and test data dropped drastically after certain numbers of epochs, even before 100 epochs. Hence, we allow an early stopping in which the training of the model terminates after the training accuracy does not reach a new high for 5 consecutive epochs. This early stopping results in better performance of DivideMix on our datasets. Regarding Co-teaching+, we use the default learning rate of 0.001 and the default number of epochs, which is 200. The batch size setting is similar to that of DivideMix.

For both of our deep-learning benchmarks, DivideMix and Co-teaching+, we use the ResNet model as the underlying network, as it has been shown as an effective network architecture for tabular data in the work of \citet{gorishniy2021revisiting}, which also provided an implementation\footnote{\url{https://github.com/yandex-research/rtdl-revisiting-models}} of the model. Our ResNet 
model consists of 10 blocks where each block consists of a batch normalization layer, a linear layer of size ($n\_features*2, n\_features*2$), a ReLU layer, a droput layer of rate 0.2, another linear layer of the same size and another dropout layer. These 10 blocks are followed by a batch normalization layer, a ReLU layer, and finally a linear layer of size ($n\_features*2$, $2$), as our data consists of two classes. 

Regarding the Confident Learning method, the method can be coupled with any classification model. Here, following the tutorial of Confident Learning on CleanLab,
we use the randomized decision trees, also known as extra-trees, as the underlying classification model. We choose the number of trees in the model to be 100.

\subsection{Datasets}\label{subsection:datasets}

We test CHNC and benchmark methods on both synthetic and real datasets. Synthetic datasets are generated using Scikit-Learn \citep{scikit-learn}. The generation of synthetic data involves a number of parameters as indicated in Table \ref{tab:syntheticdata}. 
For details on these parameters, see \citet{scikit-learn}\footnote{\url{https://scikit-learn.org/stable/modules/generated/sklearn.datasets.make_classification.html}}. There are $540$ configurations of these parameters. For each configuration, we generate $4$ different datasets, resulting in $2160$ synthetic datasets. In terms of real datasets, we experimented on $10$ datasets from the UCI Machine Learning Repository \citep{Dua:2019} listed in Table \ref{tab:realdata}.

\begin{table}
\centering
\caption{Synthetic Data Configurations}
\label{tab:syntheticdata}
\vskip 0.1in
\begin{small}
\begin{sc}
\begin{tabular}{ll}
\hline
Parameter & Values \\
\hline
\# of samples & 1000, 5000, 10000\\
\# of features & 5, 10, 20 \\ 
\% of positive samples & 30\%, 40\%, 50\%, 60\%, 70\% \\
\# of clusters per class & 2, 4 \\
class separation & 0.5, 1, 2\\
centroids on & hypercube, polytope \\
\hline
\end{tabular}
\end{sc}
\end{small}
\end{table}

\begin{table}
\centering 
\caption{Data from the UCI ML repository}
\label{tab:realdata}
\vskip 0.1in
\begin{small}
\begin{sc}
\begin{tabular}{lrrr}
\hline
Name &  Size & \# Features & \%Pos \\
\hline
Vote & 435 & 16 & 61.38 \\
Breast Cancer & 569 & 30 & 37.26 \\
Maternal Health & 1014 &  6 &  59.96 \\ 
Red Wine & 1599 & 11 & 53.47\\
Obesity & 2111 & 19 & 46.04\\
Mushroom & 8124 & 112 & 51.80\\
Phishing Websites & 11055 & 30 & 55.69\\
Dry Bean & 13611 & 17 & 45.42\\
Letter & 20000 & 16 & 49.70\\
Adult & 45222 & 82 & 24.78\\
\hline
\end{tabular}
\end{sc}
\end{small}
\end{table}

\subsection{Experiment Settings}\label{subsection:exp-setting}
For each dataset, we partition data samples into labeled set, which contains $80\%$ of the samples, and unlabeled set, which contains the remaining $20\%$. The classification performance is evaluated on the class prediction of these unlabeled samples.  

To evaluate the models' robustness to label noise, we add noise to data by changing the labels of some labeled samples. Noise levels used in this experiment are $20\%$ and $30\%$ where $x\%$ noise means we change the labels of $x\%$ of samples in each class to the opposite label. Each model uses this corrupted labeled set to train the model and predict the labels of unlabeled samples. For real data, we run $5$ experiments for each of them at each noise level using different randomization of label corruption and different partitions into labeled-unlabeled sets.

\subsection{Experiments on Synthetic Data}

\subsubsection{Evaluation metrics}

For synthetic data, we summarize the comparison between CHNC and benchmark models across all 2160 datasets using the accuracy improvement.
Suppose we denote the accuracy of the model $M$ on a particular dataset $D$ by $Acc^{(D)}(M)$, then the accuracy improvement of CHNC over $M$ on dataset $D$ is equal to $(\frac{Acc^{(D)}(CHNC)}{Acc^{(D)}(M)}-1)\times 100\%$. A positive accuracy improvement implies a higher accuracy of CHNC over the model $M$. 

To test for the statistical significance of the results, we use the Wilcoxon signed-ranks test, which is commonly used in many works such as \citet{abellan2012bagging, saez2016inffc,luengo2018cnc}. It is an appropriate statistical test when we compare two classifiers over multiple datasets, and when we do not rely on the assumption that the differences are normally distributed \citep{demvsar2006statistical}. We also report the p-values of the paired t-test, which, despite the assumptions that it relies on, is a common statistical test.

\subsubsection{Results}

From the experiments on synthetic datasets, we obtain $2160$ classification accuracy scores for each classifier. We compute the accuracy improvement given by Confidence HNC at both $20\%$ and $30\%$ noise levels, and plot the histograms as shown in Figures \ref{fig:acc_imp_synthetic_noise20} and \ref{fig:acc_imp_synthetic_noise30}.

\begin{figure}
\centering
\begin{subfigure}[b]{.5\linewidth}
\centering
\includegraphics[width=\linewidth]{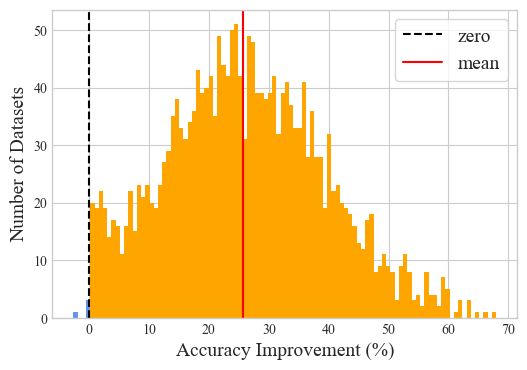}
\caption{Accuracy Improvement over DM}
\label{fig:chnc_vs_dm_20}
\end{subfigure}%
\begin{subfigure}[b]{.5\linewidth}
\centering
\includegraphics[width=\linewidth]{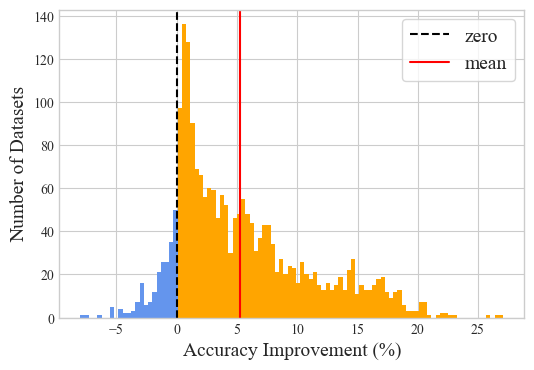}
\caption{Accuracy Improvement over CL}
\label{fig:chnc_vs_cl_20}
\end{subfigure}\\[1ex]
\begin{subfigure}{.5\linewidth}
\centering
\includegraphics[width=\linewidth]{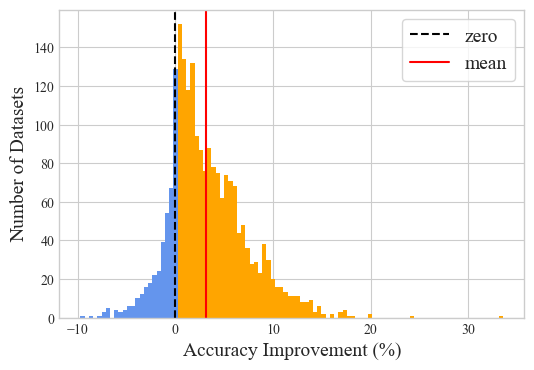}
\caption{Accuracy Improvement over CT+}
\label{fig:chnc_vs_ctp_20}
\end{subfigure}
\caption{Histograms of accuracy improvement on $2160$ synthetic datasets, with $20\%$ noise, yielded by CHNC over (a) DivideMix, (b) Confident Learning and (c) Co-teaching+. Area to the right of the dashed line indicates the datasets where CHNC outperforms. The red line indicates the mean of improvement.}
\label{fig:acc_imp_synthetic_noise20}
\end{figure}

In each plot, the dashed vertical line is at zero improvement. The area in orange on the right of this line indicates the proportion of datasets on which CHNC has positive accuracy improvement. The red solid line marks the average accuracy improvement. In Figures \ref{fig:acc_imp_synthetic_noise20} and \ref{fig:acc_imp_synthetic_noise30}, for both noise levels and for all three classifiers, the red line lies to the right of the dashed line, implying positive average improvement. These positive average accuracy improvements are shown in Table \ref{tab:synthetic_results} for both noise levels. Compared with DM, CL and CT+, CHNC outperforms in $99.62\%, 87.96\%$ and $80.31\%$ of the synthetic datasets at $20\%$ noise level, and $98.36\%, 93.19\%$ and $92.03\%$ at $30\%$ noise level, as displayed in Table \ref{tab:synthetic_results}. These percentages are the areas of the histograms that lie to the right of their zero dashed lines. 

\begin{figure}
\centering
\begin{subfigure}[b]{.5\linewidth}
\centering
\includegraphics[width=\linewidth]{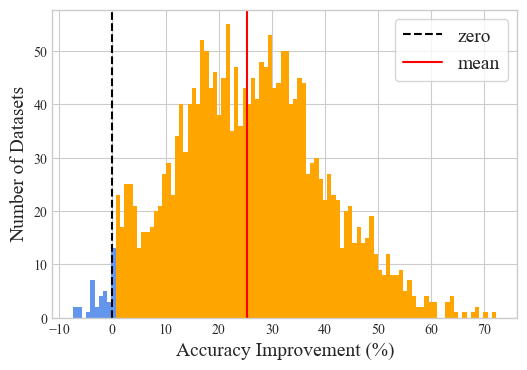}
\caption{Accuracy Improvement over DM}
\label{fig:chnc_vs_dm_30}
\end{subfigure}%
\begin{subfigure}[b]{.5\linewidth}
\centering
\includegraphics[width=\linewidth]{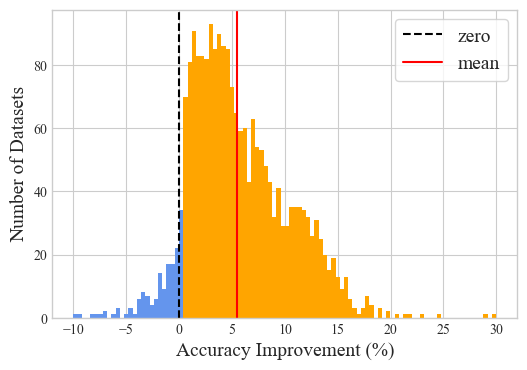}
\caption{Accuracy Improvement over CL}
\label{fig:chnc_vs_cl_30}
\end{subfigure}\\[1ex]
\begin{subfigure}{.5\linewidth}
\centering
\includegraphics[width=\linewidth]{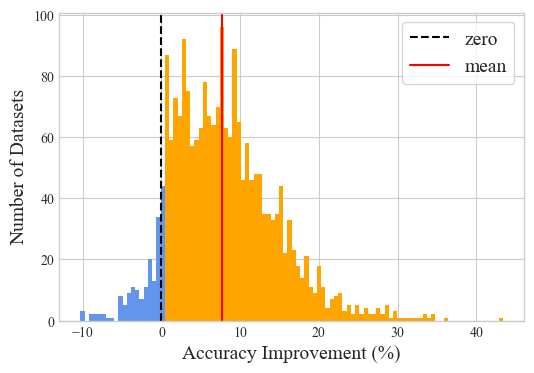}
\caption{Accuracy Improvement over CT+}
\label{fig:chnc_vs_ctp_30}
\end{subfigure}
\caption{Histograms of accuracy improvement on $2160$ synthetic datasets, with $30\%$ noise, yielded by CHNC over (a) DivideMix, (b) Confident Learning and (c) Co-teaching+. Area to the right of the dashed line indicates the datasets where CHNC outperforms. The red line indicates the mean of improvement.}
\label{fig:acc_imp_synthetic_noise30}
\end{figure}

We apply the Wilcoxon test and the paired t-test to evaluate the statistical significance of the difference between the accuracy scores of CHNC and other methods. P-values for both tests when tested with DM, CL and CT+ are close to zero, below $10^{-200}$ to be precise, at both $20\%$ and $30\%$ noise levels, validating the significance of the results.

\begin{table*}[h]
\centering
\caption{Comparisons of accuracy improvement that CHNC yields over benchmark methods on synthetic datasets for 20\% and 30\% noise levels}
\label{tab:synthetic_results}
\vskip 0.1in
\begin{small}
\begin{sc}
\begin{tabular}{llll} \hline
Noise level: 20\% & DM & CL & CT+\\ \hline
Average Acc Imp of CHNC (\%) & 25.74 & 5.23 & 3.10\\
\% data where CHNC outperforms (\%) & 99.62 & 87.96 & 80.31\\
\hline
Noise level: 30\% & DM & CL & CT+\\ \hline
Average Acc Imp of CHNC (\%) & 25.42 & 5.43 & 7.64\\
\% data where CHNC outperforms (\%) & 98.36 & 93.19 & 92.03\\
\hline
\end{tabular}
\end{sc}
\end{small}
\end{table*}

\subsection{Experiments on Real Data}

\subsubsection{Evaluation Metrics} \label{subsubsec:metrics}

We evaluate CHNC and benchmark methods in two aspects. First, we compare their classification performances. Second, we compare their capabilities in detecing noisy labels.

In terms of the classification performance, we evaluate the classification accuracy using two evaluation metrics: accuracy and balanced accuracy. Accuracy is the fraction of correct predictions among all unlabeled samples. Balanced accuracy is the average between the true positive rate and the true negative rate, and this can differ from the overall accuracy depending on the class balance in the dataset. 

Additionally, we also report for each method \textit{the average gap from the highest performance}, or \textit{Avg Gap from Max} for short. 
For a dataset $D$, we compute the gap between the performance of each method and the best performance across methods, denoted by $Acc^{(D)}_{max}$, as $\frac{Acc^{(D)}(M) - Acc^{(D)}_{max}}{Acc^{(D)}_{max}} \times 100\%$. The reported Avg Gap from Max for each method is the average of this measure across all datasets.

CHNC can also be used as a noise detector. By using CHNC, labeled samples are assigned to either prediction set simultaneously with the unlabeled samples. A labeled sample may be placed in the partition set opposite to its label if its similarity to samples with the opposite label is strong enough and the confidence penalty is small enough. Samples like these are considered noisy by the model.  Formally, with the output of CHNC being the minimum cut partition $(S,\overline{S})$ where $S$ and $\overline{S}$ are the sets of positive and negative predictions, respectively, the set $\{\overline{S} \cap L^{+}\} \cup \{S \cap L^{-}\}$ is the set of samples deemed noisy by CHNC. Regarding benchmark methods, a similar procedure can be applied, where we use the trained models to predict labeled or training samples, and labeled samples whose predicted classes are different from their given labels are considered noisy.

Denote the set of labeled samples determined as noisy by $D$, and the true set of noisy labeled samples by $N$. Noise recall is the fraction of noisy samples detected by the model, or $\frac{|N \cap D|}{|N|}$. Noise precision is the fraction of noisy samples among samples that are considered noisy by the model, or $\frac{|N \cap D|}{|D|}$. To combine both aspects of the noise detection, we summarize the noise recall and noise precision together using the noise F1 score, which is the harmonic mean of noise recall and noise precision. 

Again, we test for the significance of the results regarding both classification and noise detection performance using the Wilcoxon signed-ranks test and the paired t-test.

In addition, we inspect the confidence weights distributions of clean and noisy samples to evaluate our proposed confidence weights computation at the end of this section.

\subsubsection{Results on classification performance}

Classification accuracy of CHNC and baseline methods on real data are reported in Table \ref{tab:acc_noise_20} for the $20\%$ noise level. Each entry in the table is the average accuracy over $5$ runs along with the standard deviation.

\begin{table*}[h]
\centering
\caption{Classification accuracy, written as mean (and standard error)$(\%)$, on real data with $20\%$ noise and p-values of the outperformance of CHNC. The highest accuracy is in bold. (*) marks strong statistical significance.}
\label{tab:acc_noise_20}
\vskip 0.1in
\begin{small}
\begin{sc}
\begin{tabular}{lllll} \hline
Data & CHNC & DM & CL & CT+\\ \hline
Vote & \textbf{95.73} (0.61) & 94.66 (2.86)& 89.92 (1.31) & 90.99 (1.70)\\
Breast Cancer & \textbf{94.50} (1.91)& 92.86 (2.23)& 94.04 (0.68)& 90.99 (0.60)\\
Maternal Health & 75.54 (2.80)& 73.37 (2.29)& \textbf{76.00} (1.41)& 75.54 (1.32)\\
Red Wine & 71.96 (1.25)& 72.25 (2.11)& \textbf{74.00} (2.37)& 73.00 (1.49)\\
Obesity & \textbf{97.35} (0.80)& 92.37 (1.02)& 95.90 (1.12)& 97.29 (0.46)\\
Mushroom & 99.76 (0.32)& \textbf{99.89} (0.15)& 99.74 (0.21)& 99.70 (0.26)\\
Phishing Websites & \textbf{93.81} (0.70)& 90.52 (1.36)& 92.91 (0.64)& 93.80 (0.37)\\
Dry Bean & 97.24 (0.27)& 96.41 (0.34)& 97.00 (0.32)& \textbf{97.60} (0.21)\\
Letter & \textbf{95.53} (0.36)& 79.55 (1.08)& 93.92 (0.39)& 92.58 (0.75)\\
Adult & 83.28 (0.34)& 80.62 (0.52)& 82.52 (0.20)& \textbf{84.52} (0.24)\\ \hline
Avg gap from max (\%) & \textbf{-0.82} & -3.67 & -4.69 & -3.70 \\ \hline
Wilcoxon p-values & & 2e-6* & 0.0020* & 0.0571 \\ \hline
Paired t-test p-values & & 1e-5* & 0.0077* & 0.0065* \\ \hline
\end{tabular}
\end{sc}
\end{small}
\end{table*}

At the $20\%$ noise level, CHNC is among the two models with highest accuracy on all $10$ datasets, only except for Wine data. Among these data, CHNC achieves the highest accuracy on Vote, Breast Cancer, Obesity, Phishing Websites and Letter data. The overall standard deviation of the accuracy of CHNC is in the same magnitude as that of other models. P-values given by the Wilcoxon test demonstrate that CHNC performs better than both DM and CL with strong statistical significance ($\alpha=0.05$). P-values of its outperformance over CT+, which is 0.0571, is only slightly above 0.05. This also highlights the statistical significance at the level that is only slightly lower than the level of $\alpha=0.05$. P-values from the paired t-test indicates strong statistical significance of the outperformance of CHNC over all benchmark models.

In terms of balanced accuracy, reported in Table \ref{tab:bal_acc_noise_20}, CHNC achieves the best performance on $5$ out of $10$ datasets. It is the second best method on $2$ other datasets. On the remaining $3$ datasets, which are Maternal Health, Red Wine and Mushroom data, its balanced accuracies are within small margins of $0.24\%, 0.67\%$ and $0.04\%$, respectively, from the second best method.
Moreover, p-values of the outperformance of CHNC over all benchmarks, for both tests, are all smaller than 0.05, emphasizing its superiority over the benchmark models with strong statistical significance.

\begin{table*}[h]
\centering
\caption{Balanced accuracy, written as mean (and standard error)$(\%)$, on real data with $20\%$ noise and p-values of the outperformance of CHNC. The highest accuracy is in bold. (*) marks strong statistical significance.}
\label{tab:bal_acc_noise_20}
\vskip 0.1in
\begin{small}
\begin{sc}
\begin{tabular}{lllll} \hline
Data & CHNC & DM & CL & CT+\\ \hline
Vote & \textbf{96.14} (0.66) & 95.13 (2.31)& 89.55 (2.27) & 92.65 (1.92)\\
Breast Cancer & \textbf{94.04} (1.70)& 90.72 (3.19)& 93.98 (0.95)& 89.37 (1.01)\\
Maternal Health & 74.73 (3.40)& 74.97 (1.84)& 74.84 (1.64)& \textbf{77.58} (4.11)\\
Red Wine & 72.03 (1.32)& 72.70 (1.93)& \textbf{74.04} (2.26)& 72.35 (0.91)\\
Obesity & \textbf{97.28} (0.79)& 92.04 (0.99)& 95.77 (1.15)& 96.99 (0.80)\\
Mushroom & 99.77 (0.31)& \textbf{99.89} (0.16)& 99.74 (0.21)& 99.81 (0.12)\\
Phishing Websites & \textbf{93.79} (0.59)& 90.13 (1.42)& 92.88 (0.61)& 93.20 (0.48)\\
Dry Bean & 97.29 (0.26)& 96.51 (0.31)& 97.06 (0.32)& \textbf{97.46} (0.29)\\
Letter & \textbf{95.52} (0.37)& 79.56 (1.07)& 93.92 (0.39)& 91.91 (0.59)\\
Adult & 78.74 (0.35)& 74.99 (0.45)& 77.05 (0.52)& \textbf{79.13} (0.39)\\ \hline
Avg gap from max (\%) & \textbf{-0.72} & -4.20 & -1.78 & -1.56 \\ \hline
Wilcoxon p-values & & 4e-6* & 0.0008* & 0.0115* \\ \hline
Paired t-test p-values & & 1e-5* & 0.0057* & 0.0122* \\ \hline
\end{tabular}
\end{sc}
\end{small}
\end{table*}

For both accuracy and balanced accuracy, it has the smallest average gap from the method with best performance, averaged across all datasets, as shown at the bottom of Table \ref{tab:acc_noise_20} and \ref{tab:bal_acc_noise_20}, alongside the p-values. This gap of CHNC is distinctively smaller than that of other methods. 

As the noise goes up to $30\%$ (Table \ref{tab:acc_noise_30} and \ref{tab:bal_acc_noise_30}), CHNC is always among the two best methods for both metrics, accuracy and balanced accuracy, except for only one case, which is when we compare the balanced accuracy on Red Wine data. CHNC achieves the highest accuracy on $4$ datasets and the highest balanced accuracy on $5$ datasets out of $10$. On instances where CHNC is the second best method such as on Obesity and Dry Bean data, for both metrics, its performance is particularly close to the best performance, all within $0.5\%$. When comparing the balanced accuracy on Red Wine data, which is the only case where CHNC is not among the two best methods, the balanced accuracy of CHNC is only $0.15\%$ below the second best method and $0.75\%$ below the best one.  

P-values of the outperformance of CHNC over all benchmarks across all datasets, for both evaluation metrics and both Wilcoxon test and paired t-test, are all smaller than 0.05, demonstrating improvements over benchmark methods with strong statistical significance. The gap from the best performance, averaged across all data, for CHNC is clearly smaller than that of other benchmarks, for both accuracy and balanced accuracy.

Note that we also conduct another set of statistical tests by omitting Letter data as it appears that all benchmark methods suffer from a higher level of noise, of $30\%$, particularly on this data.  
The result is that, all p-values are still below $0.05$, except for the p-values of the Wilcoxon test on the comparison between the accuracy of CHNC and CT+ as shown in the second to last row of Table \ref{tab:acc_noise_30} where the p-value is $0.0833$, which still reflects the superiority of CHNC only that it is not at the significance level of $\alpha=0.05$.


\begin{table*}[h]
\centering
\caption{Classification accuracy, written as mean (and standard error)$(\%)$, on real data with $30\%$ noise and p-values of the outperformance of CHNC. The highest accuracy is in bold. (*) marks strong statistical significance.}
\label{tab:acc_noise_30}
\vskip 0.1in
\begin{small}
\begin{sc}
\begin{tabular}{lllll} \hline
Data & CHNC & DM & CL & CT+\\ \hline
Vote & \textbf{92.82} (1.50)& 91.60 (1.99)& 84.12 (2.07)& 86.11 (4.14)\\
Breast Cancer & 92.51 (1.59)& \textbf{95.09} (1.84)& 92.05 (1.41)& 88.19 (1.07)\\
Maternal Health & \textbf{69.84} (3.26)& 67.48 (4.20)& 64.52 (5.23)& 68.59 (4.37)\\
Red Wine & 72.25 (1.75)& 71.46 (2.65)& \textbf{72.71} (2.42)& 72.00 (1.31)\\
Obesity & 94.23 (1.14)& 91.92 (0.67)& 90.57 (2.01)& \textbf{94.61} (1.57)\\
Mushroom & \textbf{99.79} (0.33)& 98.65 (0.91)& 98.25 (0.56)& 98.15 (0.55)\\
Phishing Websites & 91.13 (1.06)& 90.36 (1.36)& 85.96 (0.72)& \textbf{92.63} (0.57)\\
Dry Bean & 97.17 (0.33)& 96.47 (0.41)& 95.15 (0.17)& \textbf{97.53} (0.16)\\
Letter & \textbf{94.23} (0.53)& 77.47 (0.83)& 88.34 (0.60)& 77.16 (3.10)\\
Adult & 81.80 (0.38)& 79.73 (0.30)& 78.79 (0.13)& \textbf{83.91} (0.32)\\ \hline
All data & & & & \\ \hline
\; Avg gap from max (\%) & \textbf{-0.82} & -3.67 & -4.69 & -3.70 \\ \hline
\; Wilcoxon p-values & & 1e-4* & 2e-8* & 0.0090* \\ \hline
\; Paired t-test p-values & & 0.0008* & 2e-10* & 0.0011* \\ \hline
Without LETTER data & & & & \\ \hline
\; Avg gap from max (\%) & \textbf{-0.91} & -2.10 & -4.52 & -2.10 \\ \hline
\; Wilcoxon p-values & & 0.0018* & 2e-7* & 0.0833 \\ \hline
\; Paired t-test p-values & & 0.0077* & 1e-8* & 0.0185* \\ \hline
\end{tabular}
\end{sc}
\end{small}
\end{table*}

\begin{table*}[h]
\centering
\caption{Balanced accuracy, written as mean (and standard error)$(\%)$, on real data with $30\%$ noise and p-values of the outperformance of CHNC. The highest accuracy is in bold. (*) marks strong statistical significance.}
\label{tab:bal_acc_noise_30}
\vskip 0.1in
\begin{small}
\begin{sc}
\begin{tabular}{lllll} \hline
Data & CHNC & DM & CL & CT+\\ \hline
Vote & \textbf{93.06} (1.50)& 92.06 (1.99)& 84.51 (2.07)& 88.88 (4.14)\\
Breast Cancer & 92.64 (1.59)& \textbf{93.94} (1.84)& 92.26 (1.41)& 86.74 (1.07)\\
Maternal Health & \textbf{68.50} (3.26)& 64.62 (4.20)& 64.15 (5.23)& 60.97 (4.37)\\
Red Wine & 72.10 (1.75)& 72.25 (2.65)& \textbf{72.85} (2.42)& 72.23 (1.31)\\
Obesity & 94.21 (1.14)& 91.85 (0.67)& 90.39 (2.01)& \textbf{94.50} (1.57)\\
Mushroom & \textbf{99.80} (0.33)& 98.67 (0.91)& 98.26 (0.56)& 98.26 (0.55)\\
Phishing Websites & 91.24 (1.06)& 90.01 (1.36)& 87.01 (0.72)& \textbf{92.10} (0.57)\\
Dry Bean & 97.19 (0.33)& 96.56 (0.41)& 95.22 (0.17)& \textbf{97.33} (0.16)\\
Letter & \textbf{94.22} (0.53)& 77.47 (0.83)& 88.33 (0.60)& 77.12 (3.10)\\
Adult & \textbf{78.31} (0.38)& 74.52 (0.30)& 74.33 (0.13)& 78.22 (0.32)\\ \hline
All data & & & & \\ \hline
\; Avg gap from max (\%) & \textbf{-0.38} & -3.72 & -4.22 & -4.38 \\ \hline
\; Wilcoxon p-values & & 8e-5* & 4e-10* & 0.0006* \\ \hline
\; Paired t-test p-values & & 0.0004* & 6e-10* & 0.0002* \\ \hline
Without LETTER data & & & & \\ \hline
\; Avg gap from max (\%) & \textbf{-0.42} & -2.16 & -4.00 & -2.85 \\ \hline
\; Wilcoxon p-values & & 0.0015* & 1e-8* & 0.0086* \\ \hline
\; Paired t-test p-values & & 0.0067* & 4e-8* & 0.0040* \\ \hline
\end{tabular}
\end{sc}
\end{small}
\end{table*}

\subsubsection{Results on noise detection performance}
We evaluate the methods' capability in detecting noise through the F1 score as explained previously in Section \ref{subsubsec:metrics}. The results are displayed in Table \ref{tab:noise_f1_20} and \ref{tab:noise_f1_20}. 

At the noise level of $20\%$, CHNC is consistently among the best two methods, except for Red Wine data, while other methods attain the lowest performance on different datasets. At the noise level of $30\%$, this argument also holds for most data, except for Breast Cancer, Red Wine and Adult data. On many datasets where CHNC is the second best method, its performance is close to the best model and significantly exceeds the other two benchmarks. 

For both noise levels, the gap from the highest F1 score for noise detection averaged across all datasets of CHNC is the smallest. The statistical tests, both Wilcoxon and paired t-test, shows that CHNC has higher F1 scores regarding noise detection than all other models at noise level of $20\%$ with high statistical significance, at the level of $\alpha=0.05$. For noise level of $30\%$, the statistical significance only shows for the comparison with CL.

Again, we conduct another set of tests with the Letter data omitted. For $20\%$ noise level, the average gap from the best performance of CHNC is smaller than other methods. Wilcoxon test demonstrates the significance in the superiority of CHNC over DM while paired t-test exhibits this for the comparison with both DM and CL. For $30\%$ noise level, this is the case for the comparison with CL.

\begin{table*}[h]
\centering
\caption{F1 score of noise detection, written as mean (and standard error)$(\%)$, on real data with $20\%$ noise and p-values of the outperformance of CHNC. The highest accuracy is in bold. (*) marks strong statistical significance.}
\label{tab:noise_f1_20}
\vskip 0.1in
\begin{small}
\begin{sc}
\begin{tabular}{lllll} \hline
Data & CHNC & DM & CL & CT+\\ \hline
Vote & 86.64 (2.21)& \textbf{88.37} (1.14)& 74.36 (5.74)& 81.37 (3.49)\\
Breast Cancer & \textbf{88.80} (4.56)& 86.91 (7.13)& 88.13 (2.21)& 81.86 (4.96)\\
Maternal Health & 56.35 (2.52)& 49.88 (4.14)& \textbf{58.20} (4.01)& 54.69 (2.26)\\
Red Wine & 51.95 (2.48)& 52.42 (2.80)& \textbf{55.00} (1.81)& 54.39 (2.61)\\
Obesity & 92.64 (1.80)& 82.75 (3.52)& 87.71 (0.85)& \textbf{94.00} (1.62)\\
Mushroom & 99.55 (0.42)& \textbf{99.60} (0.55)& 95.35 (0.83)& 99.38 (0.33)\\
Phishing Websites & \textbf{86.62} (1.04)& 78.76 (1.73)& 81.24 (1.06)& 86.02 (0.84)\\
Dry Bean & 93.78 (0.18)& 91.13 (0.40)& 90.47 (0.57)& \textbf{94.21} (0.20)\\
Letter & \textbf{91.26} (0.66)& 60.96 (0.73)& 83.55 (0.63)& 84.24 (1.35)\\
Adult & 63.01 (0.38)& 61.80 (0.52)& 62.16 (0.34)& \textbf{68.39} (0.27)\\ \hline
All data & & & & \\ \hline
\; Avg gap from max (\%) & \textbf{-2.05} & -8.82 & -5.53 & -3.15 \\ \hline
\; Wilcoxon p-values & & 1e-5* & 3e-6* & 0.0438* \\ \hline
\; Paired t-test p-values & & 5e-5* & 7e-6* & 0.0281*\\ \hline
Without LETTER data & & & & \\ \hline
\; Avg gap from max (\%) & \textbf{-2.27} & -6.12 & -5.21 & -2.64 \\ \hline
\; Wilcoxon p-values & & 0.0002* & 5e-5*& 0.2474 \\ \hline
\; Paired t-test p-values & & 0.0002* & 0.0002* & 0.1816 \\ \hline
\end{tabular}
\end{sc}
\end{small}
\end{table*}

\begin{table*}[h]
\centering
\caption{F1 score of noise detection, written as mean (and standard error)$(\%)$, on real data with $30\%$ noise and p-values of the outperformance of CHNC. The highest accuracy is in bold. (*) marks strong statistical significance.}
\label{tab:noise_f1_30}
\vskip 0.1in
\begin{small}
\begin{sc}
\begin{tabular}{lllll} \hline
Data & CHNC & DM & CL & CT+\\ \hline
Vote & 79.44 (6.96)& \textbf{88.06} (2.77)& 67.48 (7.22)& 75.57 (4.05)\\
Breast Cancer & 82.45 (8.48)& \textbf{92.13} (3.62)& 85.46 (2.15)& 79.68 (3.70)\\
Maternal Health & 54.77 (4.75)& 53.12 (4.00)& 49.25 (4.00)& \textbf{55.37} (3.23)\\
Red Wine & 57.94 (2.07)& 60.46 (1.71)& 56.53 (1.82)& \textbf{62.04} (1.47)\\
Obesity & 88.54 (2.86)& 86.24 (0.98)& 79.20 (1.87)& \textbf{90.81} (2.90)\\
Mushroom & \textbf{99.73} (0.28)& 98.02 (1.30)& 89.56 (1.00)& 97.06 (0.97)\\
Phishing Websites & 86.41 (1.74)& 84.49 (1.67)& 72.43 (0.72)& \textbf{88.16} (0.62)\\
Dry Bean & 95.37 (0.32)& 93.81 (0.53)& 87.16 (0.80)& \textbf{95.68} (0.16)\\
Letter & \textbf{91.47} (0.83)& 67.36 (0.55)& 75.96 (0.67)& 67.09 (3.11)\\
Adult & 65.67 (0.75)& 69.73 (0.22)& 60.36 (0.14)& \textbf{75.78} (0.31)\\ \hline
All data & & & & \\ \hline
\; Avg gap from max (\%) & \textbf{-4.62} & -5.38 & -13.76 & -5.7 \\ \hline
\; Wilcoxon p-values & & 0.5419 & 1e-9* & 0.5305 \\ \hline
\; Paired t-test p-values & & 0.2809 & 5e-10* & 0.1559 \\ \hline
Without LETTER data & & & & \\ \hline
\; Avg gap from max (\%) & -5.13 & \textbf{-3.05} & -13.40 & -3.38 \\ \hline
\; Wilcoxon p-values & & 0.9218 & 3e-8* & 0.9251 \\ \hline
\; Paired t-test p-values & & 0.9551 & 3e-8* & 0.8592 \\ \hline
\end{tabular}
\end{sc}
\end{small}
\end{table*}

\subsubsection{Confidence Weights Distribution}

Additionally, we assess our confidence weights computation method by examining the unscaled confidence weights of labeled samples, which vary between $0$ and $1$. Confidence weights distribution of labeled samples in three datasets: Vote, Obesity and Phishing Websites, for the noise level of $30\%$, are shown in Figure \ref{fig:confidence_score_vote}, \ref{fig:confidence_score_obesity} and \ref{fig:confidence_score_phishing}. Each plot consist of two histograms, one for the confidence weights of ``clean'' samples, or samples whose labels are not contaminated, and ``noisy'' samples whose labels are flipped. The goal is to have the two distributions to be as separable as possible, and to have the distribution of clean samples close to one and that of noisy samples close to zero.

As shown in the figures, most clean samples have their confidence weights (in blue, no pattern) higher on average than those of the samples with noisy labels (in pink, with pattern). The stark difference between the two distributions in each plot shows the efficiency of our method in translating the nested partition sequence of samples into their confidence scores.

\begin{figure}
     \centering
     \begin{subfigure}[b]{\linewidth}
         \centering
         \includegraphics[width=0.75\linewidth]{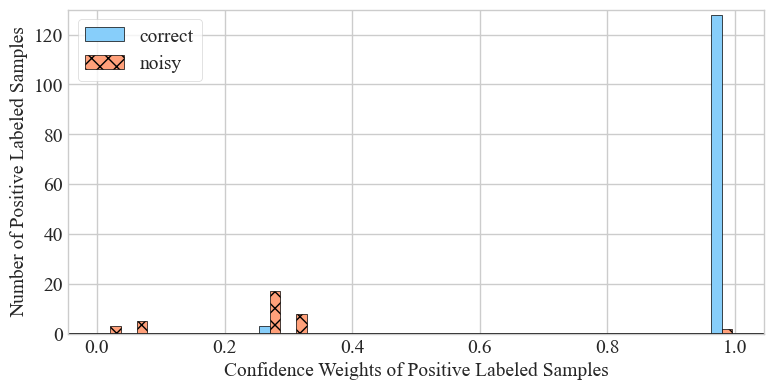}
         \caption{Positive labeled samples in Vote data}
         \label{fig:vote-positive}
     \end{subfigure}
     \begin{subfigure}[b]{\linewidth}
         \centering
         \includegraphics[width=0.75\linewidth]{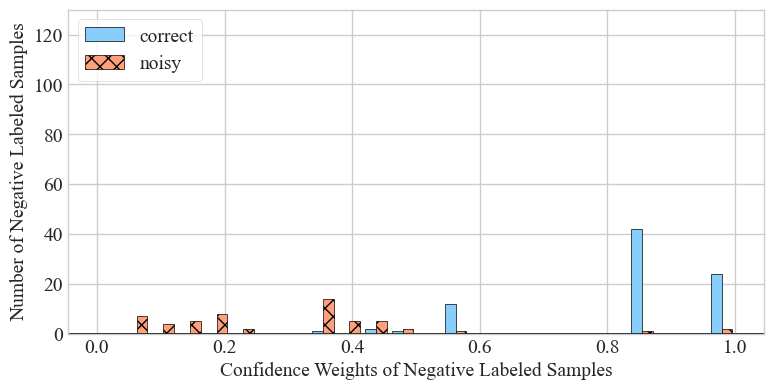}
         \caption{Negative labeled samples in Vote data}
         \label{fig:vote-negative}
     \end{subfigure}
        \caption{Confidence weights distributions of clean and noisy labeled samples in Vote data with $30\%$ noise rate}
    \label{fig:confidence_score_vote}
\end{figure}

\begin{figure}
     \centering
     \begin{subfigure}[b]{\linewidth}
         \centering
         \includegraphics[width=0.75\linewidth]{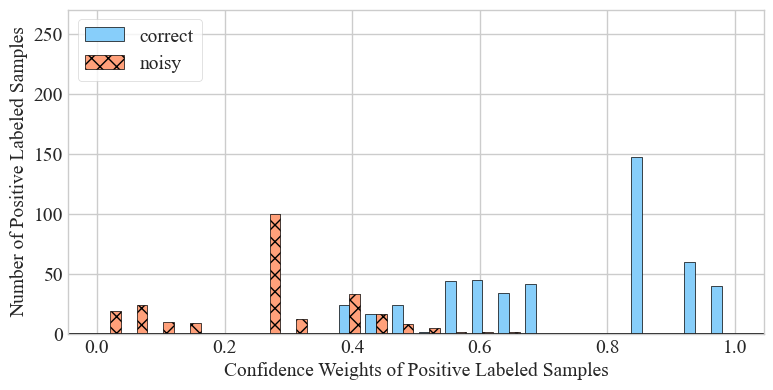}
         \caption{Positive labeled samples in Obesity data}
         \label{fig:obesity-positive}
     \end{subfigure}
     \begin{subfigure}[b]{\linewidth}
         \centering
         \includegraphics[width=0.75\linewidth]{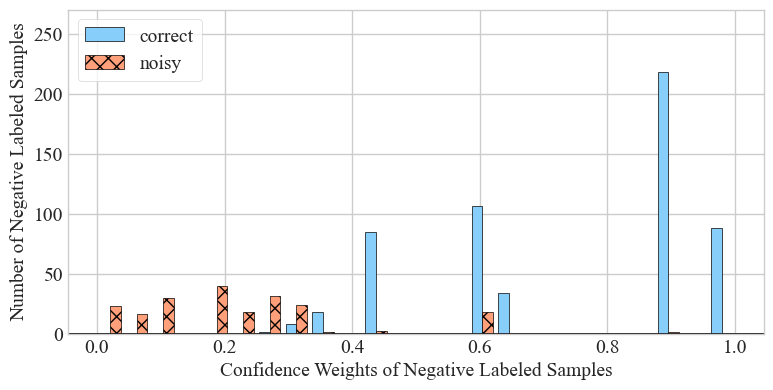}
         \caption{Negative labeled samples in Obesity data}
         \label{fig:obesity-negative}
     \end{subfigure}
        \caption{Confidence weights distributions of clean and noisy labeled samples in Obesity data with $30\%$ noise rate}
    \label{fig:confidence_score_obesity}
\end{figure}

\begin{figure}
     \centering
     \begin{subfigure}[b]{\linewidth}
         \centering
         \includegraphics[width=0.75\linewidth]{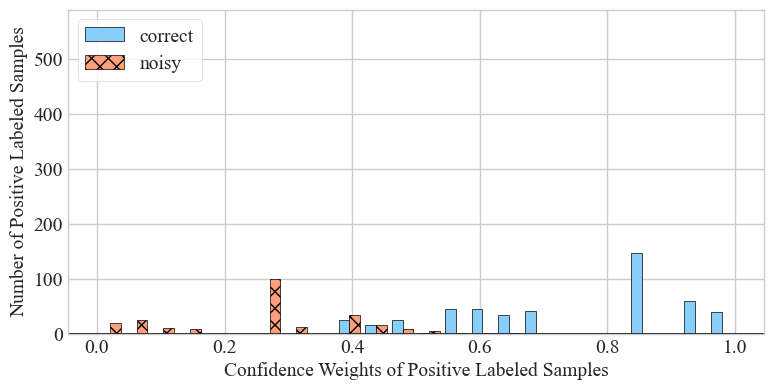}
         \caption{Positive labeled samples in Phishing data}
         \label{fig:phishing-positive}
     \end{subfigure}
     \begin{subfigure}[b]{\linewidth}
         \centering
         \includegraphics[width=0.75\linewidth]{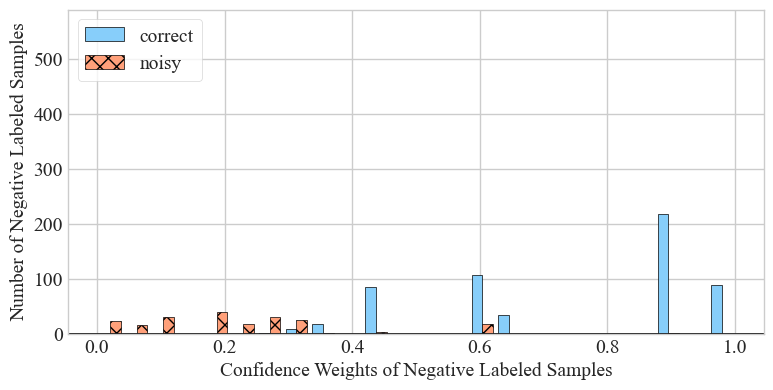}
         \caption{Negative labeled samples in Phishing data}
         \label{fig:phishing-negative}
     \end{subfigure}
        \caption{Confidence weights distributions of clean and noisy labeled samples in Obesity data with $30\%$ noise rate}
    \label{fig:confidence_score_phishing}
\end{figure}

\subsubsection{Computation time}

The computation times of CHNC and benchmark methods are shown in Table \ref{tab:time_noise_20} and \ref{tab:time_noise_30}, sorted in the ascending order of the data size. These runtimes are the results of running the experiments on the Intel Xeon Skylake 6130 CPUs. For the implementation of deep learning benchmarks which require GPUs, GTX 2080ti GPUs are used.

The runtimes of CHNC are consistently smaller than both DM and CT+ while they are comparable with those of CL for small data. 
Note that there is a slight decrease in the runtimes of CHNC, DM and CT+ as we go from the Mushroom dataset to the next larger data of size more than $10000$ samples, Phishing Websites, before they continue to rise again as we progress to other larger datasets.
This is due to the different parameter choices between data of size smaller than $10000$ and data of larger size. For CHNC, we use a smaller $k$ in the sparsification for large data, resulting in the graph representation that contains a smaller number of edges. For the two deep-learning benchmarks DM and CT+, we use a batch size of $32$ for data of fewer than $10000$ samples, and a batch size of $128$ for larger data.

\begin{table*}[h]
\centering
\caption{Computation time in seconds on real data with $20\%$ noise, averaged across 5 experiments for each dataset.}
\label{tab:time_noise_20}
\vskip 0.1in
\begin{small}
\begin{sc}
\begin{tabular}{llllll} \hline
Data & Size & CHNC & DM & CL & CT+\\ \hline
Vote & (435, 16) & 1.63 & 21.09 & 1.66 & 23.51\\
Breast Cancer & (569, 30) & 1.76 & 22.90 & 1.70 & 31.69\\
Maternal Health & (1014, 6) & 2.79 & 28.73 & 1.63 & 60.48\\
Red Wine & (1599, 11) & 4.32 & 30.17 & 2.03 & 86.56\\
Obesity & (2111, 19) & 5.83 & 35.97 & 1.93 & 114.78\\
Mushroom & (8124, 112) & 22.24 & 113.27 & 4.60 & 407.74\\
Phishing Websites & (11055, 30) & 21.15 & 49.85 & 5.07 & 151.71\\
Dry Bean & (13611, 17) &26.06 & 53.40 & 7.02 & 178.39\\
Letter & (20000, 16) & 39.62 & 95.80 & 9.37 & 266.42\\
Adult & (45222, 82) &94.37 & 192.05 & 23.74 & 590.95\\ \hline
\end{tabular}
\end{sc}
\end{small}
\end{table*}

\begin{table*}[h]
\centering
\caption{Computation time in seconds on real data with $30\%$ noise, averaged across 5 experiments for each dataset.}
\label{tab:time_noise_30}
\vskip 0.1in
\begin{small}
\begin{sc}
\begin{tabular}{llllll} \hline
Data & Size  & CHNC & DM & CL & CT+\\ \hline
Vote & (435, 16) & 1.37 & 18.79 & 1.55 & 22.71\\
Breast Cancer & (569, 30) & 1.78 & 22.86 & 2.29 & 28.92\\
Maternal Health & (1014, 6) & 2.72 & 27.73 & 1.72 & 53.31\\
Red Wine & (1599, 11) & 4.26 & 30.35 & 1.90 & 81.09\\
Obesity & (2111, 19) & 6.24 & 34.34 & 2.36 & 109.77\\
Mushroom & (8124, 112) & 22.19 & 102.89 & 4.72 & 401.52\\
Phishing Websites & (11055, 30) & 21.12 & 44.81 & 5.31 & 149.47\\
Dry Bean & (13611, 17) & 25.89 & 55.32 & 7.30 & 180.41\\
Letter & (20000, 16) & 39.30 & 96.02 & 9.63 & 273.30\\
Adult & (45222, 82) & 97.90 & 178.59 & 24.31 & 622.44\\ \hline
\end{tabular}
\end{sc}
\end{small}
\end{table*}

\section{Conclusions} \label{sec:conclusion}

We introduce a new binary classification method, Confidence HNC or CHNC, that varies its reliance on the labeled data through the use of confidence weights.
This method that balances the similarity between samples in the same prediction set versus the dis-similarity between samples from different prediction sets may reverse the given labels of some labeled samples. It is used as a classification procedure in the presence of label noise, and also as a detection mechanism of noisy labels. 
The experiments reported here demonstrate the superiority of CHNC in terms of classification accuracy and balanced accuracy, as well as the noise detection capability that is competitive in certain cases. 

Moreover, we introduced in this work a method to generate the likelihood for each labeled sample for being mislabeled through parametric minimum cut problems. This method was shown to be effective as the computed confidence scores for labeled samples correlate well with whether they are mislabeled or not.

For future works, one interesting direction is to investigate whether the framework of CHNC, which works well for noisy data, can be effective in dealing with setups in which the number of labeled samples, or size of training data, is particularly small. 
Furthermore, the newly introduced confidence weights computation may motivate a new approach on generating probability estimates for each sample for belonging to either class from a minimum cut-based approach. Typically, a minimum cut-based method determines whether a sample belongs to the positive or the negative class without quantifying the uncertainty. That is, it only indicates whether the label of a sample is $+1$ or $-1$. \citet{blum2004semi} developed a method that gives the class probability estimate for each sample using an ensemble method that solves a minimum cut problem on multiple graphs, of which each is derived by perturbing the edge weights of the original graph representation. In our work here, we demonstrate how we may infer from the parametric minimum cut solution the varying likelihood that a sample belongs to either class. An interesting extension from this idea would be to find a method can map our confidence scores to probability estimates.

\section*{Acknowledgments}
This research was supported by the AI Institute NSF Award 2112533. This research used the Savio computational cluster resource provided by the Berkeley Research Computing program at the University of California, Berkeley (supported by the UC Berkeley Chancellor, Vice Chancellor for Research, and Chief Information Officer).

\bibliographystyle{unsrtnat}  
\bibliography{ref}

\end{document}